\documentclass{article}

\usepackage{microtype}
\usepackage{graphicx}
\usepackage{subcaption}
\usepackage{booktabs} % for professional tables  
\usepackage[utf8]{inputenc}
\usepackage{textcomp}
\usepackage{siunitx}
\usepackage{longtable,tabularx}
\usepackage{algorithm}
\usepackage{comment}
\usepackage{upgreek}
\usepackage{booktabs}
\usepackage{setspace}
\usepackage{amsfonts}
\usepackage{natbib}
\usepackage{bigints}
\usepackage{xspace}
\usepackage{hyperref}
\usepackage{amsmath}
\usepackage{amssymb}
\usepackage{mathtools}
\usepackage{amsthm}
\usepackage{algorithm}
\usepackage{algorithmic}
\usepackage[capitalize,noabbrev]{cleveref}
\usepackage{booktabs}
\usepackage{pifont}
\usepackage[table]{xcolor}

\newcommand{\cmark}{\ding{51}}

\DeclareMathOperator{\tr}{tr}

\usepackage[accepted]{icml2026}

\theoremstyle{plain}
\newtheorem{theorem}{Theorem}[section]

\newtheorem{lemma}[theorem]{Lemma}

\theoremstyle{definition}

\theoremstyle{remark}
\newtheorem{remark}[theorem]{Remark}

\icmltitlerunning{PAMD: Structured Adaptive Distances for Bisimulation Representations in Visual Reinforcement Learning}

\graphicspath{{images/}}

\begin{document}

\twocolumn[
  \icmltitle{PAMD: Structured Adaptive Distances for Bisimulation Representations in Visual Reinforcement Learning}

  % It is OKAY to include author information, even for blind submissions: the
  % style file will automatically remove it for you unless you've provided
  % the [accepted] option to the icml2026 package.

  % List of affiliations: The first argument should be a (short) identifier you
  % will use later to specify author affiliations Academic affiliations
  % should list Department, University, City, Region, Country Industry
  % affiliations should list Company, City, Region, Country

  % You can specify symbols, otherwise they are numbered in order. Ideally, you
  % should not use this facility. Affiliations will be numbered in order of
  % appearance and this is the preferred way.
  \icmlsetsymbol{equal}{*}

  \begin{icmlauthorlist}
    \icmlauthor{Daegyeong Roh}{equal,lics}
    \icmlauthor{Juho Bae}{equal,lics}
    \icmlauthor{Han-Lim Choi}{lics}
  \end{icmlauthorlist}

  \icmlaffiliation{lics}{Department of Aerospace Engineering, Korea Advanced Institute of Science and Technology, Daejon, Republic of Korea}
  % \icmlaffiliation{comp}{Company Name, Location, Country}
  % \icmlaffiliation{sch}{School of ZZZ, Institute of WWW, Location, Country}

  \icmlcorrespondingauthor{Han-Lim Choi}{hanlimc@kaist.ac.kr}
  % You may provide any keywords that you find helpful for describing your
  % paper; these are used to populate the "keywords" metadata in the PDF but
  % will not be shown in the document
  \icmlkeywords{Machine Learning, ICML}

  \vskip 0.3in
]

% this must go after the closing bracket ] following \twocolumn[ ...

% This command actually creates the footnote in the first column listing the
% affiliations and the copyright notice. The command takes one argument, which
% is text to display at the start of the footnote. The \icmlEqualContribution
% command is standard text for equal contribution. Remove it (just {}) if you
% do not need this facility.

% Use ONE of the following lines. DO NOT remove the command.
% If you have no special notice, KEEP empty braces:
% \printAffiliationsAndNotice{}  % no special notice (required even if empty)
% Or, if applicable, use the standard equal contribution text:
\printAffiliationsAndNotice{\icmlEqualContribution}

\begin{abstract}
  Many visual reinforcement learning (RL) algorithms learn representations by matching latent distances to a behavioral distance induced by reward and transition similarity. In practice, the choice of the latent distance can strongly affect performance: using fixed, pre-specified global norms (e.g., $\ell_p$ norms or other hand-designed metrics) may be overly restrictive to capture the behavioral distance. In contrast, unconstrained pairwise distances may admit degenerate solutions that drive the metric loss down without improving the representation. To address this gap, we introduce \emph{PAMD: Pairwise Adaptive Mahalanobis Distance}\footnote{Code available at \url{https://github.com/daegyeongroh/PAMD}.}, which parameterizes a positive-definite, pair-conditioned metric for measuring latent state similarity. 
  PAMD is a simple plug-in for existing bisimulation-based methods, offering a more expressive yet structured alternative to fixed, pre-specified latent distances. We empirically validate our method on visual MuJoCo continuous-control tasks, where final performance of several recent bisimulation-based RL algorithms is substantially improved when equipped with the distance we propose.
\end{abstract}
\section{Introduction}

The performance of deep reinforcement learning (RL) algorithms provided with image observations is often dominated by a low-dimensional representation of the environment\cite{hafner2019dreamer,laskin2020drq,yarats2021drqv2}. Several lines of work in obtaining good representations include providing auxiliary tasks and self-supervision objectives~\cite{pathak2017curiosity}; enforcing reconstruction objectives~\cite{yarats2021improving}; and use of contrastive loss for augmentation pairs~\cite{laskin2020curl}. 

Recently, a separate line of work studies \emph{behavioral metrics} that quantify state similarity in a Markov Decision Process (MDP) through \emph{reward} and \emph{transition} similarity. Representative examples include the bisimulation-based DBC~\cite{zhang2021deep}, MICo~\cite{castro2021mico}, and SimSR~\cite{zang2022simsr}. While they share the common goal of capturing behavioral state similarity, they adopt different instantiations of the underlying behavioral pseudometric and different choices of latent-space dissimilarity used for representation learning (e.g., $\ell_p$ norms versus cosine-based measures). Despite such modeling choices, they share a common principle: two states are behaviorally similar if they yield similar immediate rewards and induce similar distributions over next states (under the same action or policy).

Formally, these approaches typically define a behavioral pseudometric $d$ as the fixed point of a Bellman-style operator that combines a reward discrepancy term with a transition-discrepancy functional:
\begin{align} \label{eq:generic_behavioral_distance}
  &d \;=\; \mathcal{T}(d), \\
  &(\mathcal{T}d)(s,s')\;=\;\mathcal{A}\!\big(\rho_r(s,s')\;+\;\gamma\,\mathcal{D}\!\left(P(\cdot|s),\,P(\cdot|s');\,d\right)\big), \nonumber
\end{align}
where $\rho_r$ measures immediate reward mismatch, $\mathcal{A}$ aggregates action dependence (e.g., a $\max_a$ operator or an on-policy expectation), and $\mathcal{D}$ quantifies discrepancy between transition distributions in a manner that may depend on $d$ (e.g., a Wasserstein/IPM-style discrepancy or simpler couplings used for tractability).

Accordingly, encoders are trained to map observations to a latent space in which a \emph{chosen} dissimilarity $\delta$ (e.g., an $\ell_p$ norm~\cite{zhang2021deep} or cosine distance~\cite{zang2022simsr}) between representations aligns with the target behavioral pseudometric:
$\delta(\phi(s_1),\phi(s_2)) \approx d(s_1,s_2)$.
Viewed this way, learning the representation (and, potentially, the latent-space dissimilarity $\delta$) can be cast as a metric embedding problem.

A natural way to formalize the above representation objective is through the lens of \emph{metric embedding}:
given a (pseudo)metric $d$ over MDP states, we seek a map $\phi : \mathcal{X}\to\mathbb{R}^m$ such that $\delta\left(\phi(s), \phi(s')\right)$ approximates $d(s,s')$ up to small distortion. This perspective is particularly relevant in deep RL since neural representations live in a finite-dimensional Euclidean space, and the quality of generalization induced by an encoder depends critically on how well the target behavioral distance can be embedded in such a space.

\cite{castro2023kernel} makes this connection explicit by developing a kernel-based theory of behavioural distances. A positive-definite kernel whose induced distance $d^{\pi}_{\mathrm{ks}}$ is introduced and shown to coincide with the \emph{reduced} MICo distance $\Pi U^\pi$ (Theorem~15). Furthermore, it is proved that $d^{\pi}_{\mathrm{ks}}$ (equivalently, $\Pi U^\pi$) admits a low-distortion embedding into a finite-dimensional Euclidean space (Theorem~19). These results highlight that, beyond defining a behavioural distance, representation learning hinges on the interaction between the target behavioural geometry and the \emph{prescribed} latent dissimilarity used to embed it.

Yet, existing behavioural-metric methods typically treat the latent dissimilarity $\delta$ as a design choice fixed a priori (e.g., $\ell_p$ norms~\cite{zhang2021deep} or MICo distance~\cite{castro2021mico}).
Consequently, they rely on an implicit assumption: the target behavioural pseudometric is sufficiently well-\emph{embeddable} under that particular geometry. This leaves open a complementary lever---\emph{choosing or adapting} $\delta$ itself---which we explore in this work.

Motivated by this perspective, we ask a complementary question: rather than committing to a single, fixed choice of $\delta$, can we \emph{shape} the latent dissimilarity to improve embeddability for the behavioral distances encountered in high-dimensional visual RL? In principle, one could make $\delta$ fully learnable via an unconstrained pairwise network to maximize flexibility. In practice, however, such flexibility can admit degenerate solutions in which the distance head $\delta$ absorbs the supervision and matches the target pseudometric without inducing a correspondingly structured encoder, undermining the intended behavioral organization of the representation. This suggests that improving embeddability requires not arbitrary flexibility, but a carefully chosen inductive bias on $\delta$.

In this work, we propose \emph{PAMD: Pairwise Adaptive Mahalanobis Distance}, a
\emph{structured yet adaptive} latent dissimilarity $\delta$ parameterized as the
square root of a \emph{pair-conditioned positive-definite quadratic form}:
\begin{equation}
  \delta(z,z') \;=\; \sqrt{(z - z')^\top G(z,z') (z - z')}, \quad G(z,z')\succ 0.
\end{equation}
This can be viewed as a pairwise-conditioned local Mahalanobis geometry; Appendix~\ref{app:geometric_motivation} provides an additional geometric interpretation of this pair-conditioned quadratic form as a local surrogate for spatially varying latent geometry. This choice occupies a natural middle ground between fixed $\ell_p$ geometries and fully learned similarity heads: it is expressive enough to capture context-dependent anisotropy, yet sufficiently structured to keep the representation learning problem well-posed. We empirically show that this design improves both metric fitting and downstream RL performance across visual control tasks, and we provide ablations demonstrating that the proposed structure reduces representation under-training\textemdash{}where the distance module matches the metric target while the encoder remains weakly shaped\textemdash{}compared to unconstrained learned dissimilarities.

\section{Related Work}
\label{sec:related_work}

\paragraph{Behavioral distances and fixed-point targets.}
Bisimulation admits a behavioral-distance view as the fixed point of a Bellman-style operator, providing a graded notion of state similarity beyond binary equivalence \cite{ferns2004metrics}.
Later work studied policy-dependent variants that define the target distance under a fixed policy, while preserving a contractive fixed-point structure \cite{castro2020scalable}.

\paragraph{Metric embedding under fixed-point supervision.}
In deep reinforcement learning, the practical challenge is \emph{metric embedding}: learning representations in which a tractable latent comparator aligns with a fixed-point behavioral target.
DBC is an early scalable instantiation of this idea, training latent representations by regressing a fixed $\ell_1$ distance in the embedding space toward a target constructed from rewards and transition dynamics \cite{zhang2021deep}.
Subsequent work highlighted that representation quality is highly sensitive to the \emph{distance class} used for this regression.
\cite{agarwal2021robustbisim} analyzes failure modes such as representation collapse and uncontrolled scaling, arguing that the inductive bias and scale behavior of the embedding distance can materially affect alignment with the fixed-point target and yield suboptimal embeddings.
\paragraph{Beyond norm-based embeddings: diffuse metrics and collapse-aware comparators.}
\cite{castro2021mico} introduces MICo, noting that once the behavioural target is learned from sampled transitions, enforcing a strict metric (in particular, $d(s,s)=0$) becomes structurally mismatched with the stochastic, sample-based fixed-point target.
They therefore adopt a \emph{diffuse} distance that allows non-zero self-distance and implement it with the MICo distance
\begin{equation}
U_(s, s')
\;=\;
\frac{ \lVert \phi(s) \rVert_2^2 + \lVert \phi(s') \rVert_2^2 }{2}
\;+\;
\beta \, \theta\!\bigl(\phi(s), \phi(s')\bigr),
\label{eq:mico_pair_distance}
\end{equation}
where the norm terms act as learnable self-distances. This formulation improves compatibility with diffuse behavioral targets by relaxing the strict zero self-distance constraint. This line of work emphasizes that the fixed-point target may not align with an embedding comparator that enforces $d(s,s)=0$, and that relaxing this constraint can broaden representational expressivity.

\cite{zang2022simsr} proposes SimSR, which instantiates a MICo-style Bellman fixed-point operator with a cosine-distance comparator,
\begin{equation}
d_{\cos}(\phi(s),\phi(s')) 
= 1 - \frac{\phi(s)^\top \phi(s')}{\|\phi(s)\|\,\|\phi(s')\|}
\label{eq:simsr_pair_distance}
\end{equation}

thereby constraining embeddings to the unit sphere and ensuring zero self-distance by design, aiming to mitigate representation collapse while preserving the same fixed point as MICo.
\section{Preliminaries} 
\paragraph{MDP.} 
We consider a discounted Markov decision process (MDP) $\mathcal{M}=(\mathcal{S},\mathcal{A},P,r,\gamma)$, where $\mathcal{S}$ is the state space, $\mathcal{A}$ is the action space, $P(\cdot\mid s,a)$ is the transition kernel, $r(s,a)$ is the expected immediate reward, and $\gamma\in[0,1)$ is the discount factor. At time $t$, the agent observes $s_t\in\mathcal{S}$, samples an action $a_t\sim\pi(\cdot\mid s_t)$ under a policy $\pi$, receives reward $r_t=r(s_t,a_t)$, and transitions to $s_{t+1}\sim P(\cdot\mid s_t,a_t)$. The agent’s objective is to maximize the expected discounted return $\max_\pi \mathbb{E}_\pi\!\left[\sum_{t=0}^{\infty}\gamma^t r(s_t,a_t)\right]$. In our visual setting, we treat each $s\in\mathcal{S}$ as the agent’s observation. 

\paragraph{Notations.} We use $\|\cdot\|_p$ to denote $\ell_p$ norms, and $\langle\cdot,\cdot\rangle$ to denote the standard inner product. We define the policy-induced transition and reward as 
\[
\begin{aligned}
P^\pi(\cdot\mid s)
&\coloneqq \mathbb{E}_{a\sim\pi(\cdot\mid s)}
\big[P(\cdot\mid s,a)\big],\\
r^\pi(s)
&\coloneqq \mathbb{E}_{a\sim\pi(\cdot\mid s)}
\big[r(s,a)\big].
\end{aligned}
\]
and write $\mathbb{E}_\pi[\cdot]$ for expectations over trajectories generated by $\pi$ and $P$. To learn compact features from high-dimensional observations, we use an encoder $\phi:\mathcal{S}\to\mathbb{R}^n$ and denote the latent representation by $z=\phi(s)$ (and $z_t=\phi(s_t)$ along a trajectory). For sampled data (e.g., off-policy training), we denote by $\mathcal{D}$ a replay buffer or dataset of transitions $(s_t,a_t,r_t,s_{t+1})$, and by $B\sim\mathcal{D}$ a mini-batch drawn from it.
\paragraph{Behavioral distances as fixed points.}
A common thread in \emph{bisimulation-style} approaches is to define a behavioral (pseudo)metric as the unique fixed point of a Bellman-style operator acting on functions over state pairs, following the classical bisimulation-metric line of work~\cite{ferns2004metrics}.
Concretely, for a wide class of constructions one can write
\[
d \;=\; \mathcal{T}(d),
\]
where $\mathcal{T}$ combines (i) an immediate reward discrepancy and (ii) a discrepancy between transition distributions measured with respect to the current candidate distance.
Different behavioral distances arise from different choices of (a) how action dependence is aggregated (e.g., $\max_a$ versus on-policy expectation) and (b) how transition discrepancies are measured (e.g., optimal-transport couplings versus tractable relaxations, or latent-space measurements).
Below we summarize three representative instances widely used in deep RL.

\smallskip
\noindent\textbf{Bisimulation metric.}
One canonical choice is the (action-conditional) bisimulation operator $\mathcal{T}_{\mathrm{bis}}$:
\begin{align} \label{eq:bisim_operator}
  (\mathcal{T}_{\mathrm{bis}} d)(s,s') \coloneqq \max_{a\in\mathcal{A}} \Big(&|r(s,a)-r(s',a)| \\ \nonumber
  & + \gamma\, W_d\!\big(P(\cdot|s,a),\,P(\cdot|s',a)\big) \Big),
\end{align}
where $W_d$ denotes the $1$-Wasserstein distance between transition distributions with ground cost $d$.
Under standard conditions, $\mathcal{T}_{\mathrm{bis}}$ is a $\gamma$-contraction in the $\ell_\infty$ sense and hence admits a unique fixed point, which defines the bisimulation pseudometric.

DBC~\cite{zhang2021deep} builds on this fixed-point view, but uses an on-policy counterpart rather than directly optimizing the max-action operator in Eq.~\eqref{eq:bisim_operator}.
Concretely, the action-wise maximization is replaced by reward and transition distributions induced by the current policy, yielding a policy-conditioned target of the form
\[
(\mathcal{T}_{\mathrm{bis}} d)(s,s') \coloneqq
|r^\pi(s)-r^\pi(s')|
+ \gamma\,W_d\!\big(P^\pi(\cdot|s),P^\pi(\cdot|s')\big).
\]
As the policy improves online, the resulting fixed point is interpreted as a $\pi^*$-bisimulation metric.
DBC then learns an encoder $\phi$ such that a \emph{prescribed} latent dissimilarity
(e.g., $\|\phi(s)-\phi(s')\|_1$) aligns with this on-policy bisimulation target.

\smallskip
\noindent\textbf{Independently-coupled policy distances (MICo \& SimSR).}
For a fixed policy $\pi$, a common relaxation of the bisimulation operator replaces the optimal-transport
coupling (as in $W_d$) with an \emph{independent} coupling of next states.
MICo~\cite{castro2021mico} adopts this choice and defines the update operator $\mathcal{T}^{\pi}_{\mathrm{M}}$
on $U:\mathcal{S}\times\mathcal{S}\to\mathbb{R}_{\ge 0}$:
\begin{align} \label{eq:mico_operator}
  (\mathcal{T}^{\pi}_{\mathrm{M}} U)(s,s') \coloneqq |r^\pi(s) - &r^\pi(s')| \\ \nonumber
  & + \gamma\,\mathbb{E}_{\substack{s_+\sim P^\pi(\cdot|s)\\ s'_+\sim P^\pi(\cdot|s')}}\!\big[U(s_+,s'_+)\big],
\end{align}
where $r^\pi(s):=\mathbb{E}_{a\sim\pi(\cdot|s)}[r(s,a)]$ and
$P^\pi(\cdot|s):=\mathbb{E}_{a\sim\pi(\cdot|s)}[P(\cdot|s,a)]$.
The operator $\mathcal{T}^{\pi}_{\mathrm{M}}$ is a $\gamma$-contraction and thus admits a unique fixed point
$U^\pi$, referred to as the MICo distance.

The same independently-coupled Bellman backbone can also be instantiated by specifying how the
``next-step discrepancy'' is measured, which yields practical objectives for representation learning.
SimSR~\cite{zang2022simsr}, in particular, measures next-step discrepancy in the learned latent space via a
cosine-based dissimilarity and (optionally) a latent dynamics model.
Concretely, given an encoder $\phi$ and a latent dynamics model $P^\pi_{\phi}(\cdot\,|\,\phi(s))$ over latent
next states, SimSR considers the operator
\begin{align} \label{eq:simsr_operator}
  (\mathcal{T}^{\pi}_{\cos} U)(s,s') \coloneqq |r^\pi(s) &- r^\pi(s')| \\ \nonumber 
  &+ \gamma \mathbb{E}_{\substack{z_+\sim P^\pi_{\phi}(\cdot|\phi(s))\\ z'_+\sim P^\pi_{\phi}(\cdot|\phi(s'))}}
  \big[\delta_{\cos}(z_+,z'_+)\big],
\end{align}
where $\delta_{\cos}(z,z') := 1 - \frac{\langle z,z'\rangle}{\|z\|_2\|z'\|_2}$.
When $\phi$ and the latent dynamics are specified, the corresponding operator admits a fixed point,
providing a Bellman-style behavioral objective aligned with the independently-coupled recursion above.

\paragraph{Latent dissimilarity and embedding viewpoint.}
Throughout, we measure representation similarity via a \emph{latent dissimilarity} $\delta:\mathbb{R}^n\times\mathbb{R}^n\to\mathbb{R}_{\ge 0}$ applied to $z=\phi(s)$. The structural assumptions we impose on $\delta$ are symmetry, nonnegativity, and self-zero: $\delta(z,z)=0$ for all $z$. In particular, $\delta$ need not be a true metric and may itself be a pseudometric. Given a target behavioral pseudometric $d$ on $\mathcal{S}$ and a choice of $\delta$, we view representation learning as finding an encoder $\phi$ such that $\delta(\phi(s),\phi(s'))$ approximates $d(s,s')$ with small distortion. Since $d$ is a pseudometric, it induces an equivalence relation $s\sim_d s' \Leftrightarrow d(s,s')=0$ and the corresponding quotient space $\mathcal{S}/\!\sim_d$. The associated \emph{reduced distance} $\bar d$ on $\mathcal{S}/\!\sim_d$ is well-defined by $\bar d([s],[s']) := d(s,s')$. Accordingly, the embedding objective is fundamentally about preserving geometry on the quotient: ideally, $\phi$ should identify each $0$-distance class (so that $s\sim_d s'$ implies $\delta(\phi(s),\phi(s'))\approx 0$) while remaining discriminative across distinct classes. This quotient-aware viewpoint clarifies why the choice (or adaptation) of $\delta$ can directly affect embeddability in finite-dimensional latent spaces, even when the target behavioral distance $d$ is fixed.

\section{Architecture} \label{sec:architecture}
We propose a plug-in latent distance module\textemdash{}a pairwise-conditioned positive definite (PD) quadratic-form distance\textemdash{}that directly replaces the baseline distance in bisimulation-based representation learning.
\subsection{Necessity of well-constrained distance architecture}
In bisimulation-based representation learning, the latent distance is supervised only through a
scalar fixed-point target. A fixed global norm (e.g., $\ell_1$ or $\ell_2$) is simple but overly rigid, imposing the same isotropic notion of similarity across all regions and directions of the latent space. At the other extreme, an unconstrained pairwise scalar parameterization (e.g., an MLP on $[z;z']$) is too flexible and can satisfy the objective by direct regression on latent pairs, without inducing a structured distance that depends on relative differences between latents.
\begin{figure}[t]
  \vskip 0.2in
  \includegraphics[width=82mm]{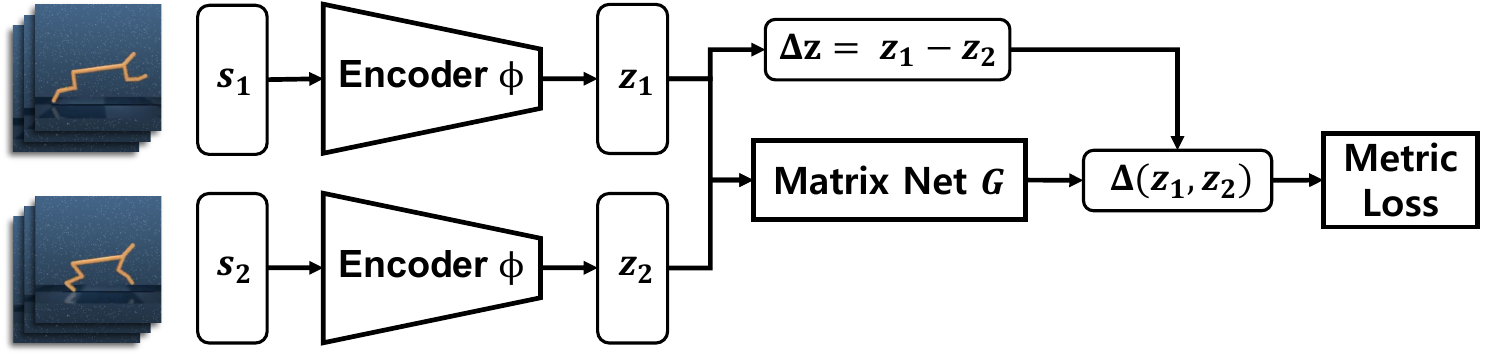}
  \caption{Architecture summary.} \label{fig:arch}
\end{figure}
We therefore adopt a pairwise-conditioned PD quadratic-form distance defined on the displacement $\Delta \coloneqq z - z'$,
which guarantees nonnegativity and symmetry by construction, while allowing adaptive, input-dependent anisotropy.

\subsection{Pairwise-Conditioned PD Metric Network}
We parameterize a pairwise-conditioned PD matrix with a lightweight MLP (\texttt{MetricNet}). Given $z,z'\in\mathbb{R}^p$, the network takes $[z;z']\in\mathbb{R}^{2p}$ and outputs $k=\frac{p(p+1)}{2}$ scalars, which fill the lower-triangular entries of $L_\theta(z,z')\in\mathbb{R}^{p\times p}$ with network parameters $\theta$. From the Cholesky decomposition, we construct
\begin{equation}
  G_0(z,z') \;=\; L_\theta(z,z')\,L_\theta(z,z')^\top,
\end{equation}
which is PD by construction, and enforce symmetry via
\begin{equation} \label{eq:symmetrize}
  G_\theta(z,z') \;=\; G_0(z,z') + G_0(z',z).
\end{equation}
To avoid singularities, we add a small ridge
\begin{equation} \label{eq:ridge}
  \bar G_\theta(z,z') \;=\; G_\theta(z,z') + \lambda I,
\end{equation}
where $\lambda>0$ is a small constant. We apply this quadratic form to $\Delta=z-z'$ to obtain a latent distance.

\paragraph{Induced distance through trace normalization.}
Since the distance can be satisfied up to an arbitrary global scale, we fix this degree of freedom
via per-pair trace normalization:
\begin{equation}
  \tilde G_\theta(z,z') \;\coloneqq\; \frac{\bar G_\theta(z,z')}{\tr\!\big(\bar G_\theta(z,z')\big) + \varepsilon},
\label{eq:trace_norm}
\end{equation}
where $\varepsilon>0$ is a small constant added for numerical stability.
We then define PAMD as follows:
\begin{equation} \label{eq:plugin_distance}
  d_\theta(z,z') \;=\; \sqrt{\Delta^\top\,\tilde G_\theta(z,z')\,\Delta + \varepsilon}.
\end{equation}
This distance is a drop-in replacement for the latent distance used by prior bisimulation-based objectives. The following lemma characterizes its spectral properties showing that $d_\theta(z,z') \lesssim \|\Delta\|_2$ holds, and $d_\theta(z,z') \gtrsim \|\Delta\|_2$ under a mild bounded-trace condition.
\begin{lemma}[Spectral bounds under trace normalization] \label{lem:spectral_sandwich}
  Fix $\lambda>0$ and $\varepsilon>0$. Define $G_\theta(z,z')$ and $\bar G_\theta(z,z')$ as in~\eqref{eq:symmetrize} and~\eqref{eq:ridge} respectively,
  and define $\tilde G_\theta(z,z')$ as in~\eqref{eq:trace_norm}.
  Then for any $\Delta\in\mathbb{R}^p$,
  \begin{align} \label{eq:quad_sandwich_pairwise}
    \underbrace{\frac{\lambda}{\tr(G_\theta(z,z')) + p\lambda + \varepsilon}}_{=:\,\alpha(z,z')}\,\|\Delta\|_2^2
    &\le \Delta^\top \tilde G_\theta(z,z')\,\Delta \\ \nonumber
    &\le \|\Delta\|_2^2.
  \end{align}
  Consequently, the induced distance \eqref{eq:plugin_distance} satisfies
  \begin{equation}
  \label{eq:dist_sandwich_pairwise}
  \sqrt{\alpha(z,z')\,\|\Delta\|_2^2+\varepsilon}
  \;\le\;
  d_\theta(z,z')
  \;\le\;
  \sqrt{\|\Delta\|_2^2+\varepsilon}.
  \end{equation}
  Moreover, if $\tr(G_\theta(z,z'))\le T$ holds uniformly over the pairs of interest, then
  $\alpha(z,z')\ge \alpha := \frac{\lambda}{T+p\lambda+\varepsilon}$ is a uniform lower bound.
\end{lemma}
\begin{proof}
  Fix $(z,z')$ and write $\bar G:=\bar G_\theta(z,z')$, $\tilde G:=\tilde G_\theta(z,z')$. Since $G_\theta\succ 0$ and $\lambda I\succ 0$, we have $\bar G\succ 0$ and hence $\tilde G\succ 0$.

  For any PD matrix $A$, $\lambda_{\max}(A)\le \tr(A)$. Thus
  \[
  \lambda_{\max}(\tilde G)
  =\frac{\lambda_{\max}(\bar G)}{\tr(\bar G)+\varepsilon}
  \le \frac{\tr(\bar G)}{\tr(\bar G)+\varepsilon}
  \le 1.
  \]
  Therefore $\Delta^\top \tilde G \Delta \le \lambda_{\max}(\tilde G)\|\Delta\|_2^2 \le \|\Delta\|_2^2$.

  Because $\bar G = G_\theta + \lambda I \succ \lambda I$, we have $\lambda_{\min}(\bar G)\ge \lambda$.
  Hence
  \[
  \lambda_{\min}(\tilde G)
  =\frac{\lambda_{\min}(\bar G)}{\tr(\bar G)+\varepsilon}
  \ge \frac{\lambda}{\tr(\bar G)+\varepsilon}.
  \]
  Using $\tr(\bar G)=\tr(G_\theta)+p\lambda$ gives
  \[
  \lambda_{\min}(\tilde G)\ge \frac{\lambda}{\tr(G_\theta)+p\lambda+\varepsilon}.
  \]
  Finally, $\Delta^\top \tilde G \Delta \ge \lambda_{\min}(\tilde G)\|\Delta\|_2^2$ yields \eqref{eq:quad_sandwich_pairwise},
  and \eqref{eq:dist_sandwich_pairwise} follows by adding $\varepsilon$ inside the square root as in \eqref{eq:plugin_distance}. The uniform bound is immediate if $\tr(G_\theta)\le T$.
\end{proof}
\begin{remark}[Preservation of theoretical properties]
  Since we do not modify the definition of the target behavioral distance (nor its defining operator), theoretical guarantees stated purely in terms of that distance remain applicable. In particular, fixed-point characterizations and value-function relations expressed in terms of the behavioral distance carry over unchanged. A concise mapping to representative results for multiple bisimulation operators is provided in Table~\ref{tab:preserved_properties} in Appendix~\ref{app:mapping_prior_results}.
\end{remark}
\paragraph{Metric loss.}
We recall the generic expression~\eqref{eq:generic_behavioral_distance} of behavioral distances as fixed points, which is instantiated in~\eqref{eq:mico_operator} and~\eqref{eq:simsr_operator} for MICo and SimSR respectively. We write the metric loss used to train the distance module as
\begin{align} \label{eq:metric_loss}
  \mathcal{L}_{\mathrm{metric}}(\theta, \; &\omega) = \\
  &\mathbb{E}_{(s,s')\sim B} \Big[\big(d_\theta(\phi_\omega(s),\phi_\omega(s')) - \mathcal{T}d_\theta(s,s')\big)^2\Big], \nonumber
\end{align}
where $B\sim\mathcal{D}$ is a mini-batch of state pairs sampled from the replay buffer or dataset $\mathcal{D}$, and $\mathcal{T}$ is the operator which defines the target behavioral distance.
\section{Experiments} \label{sec:experiments}
We validate the efficacy of PAMD as a drop-in replacement for the latent comparator used in bisimulation-based visual RL on pixel-based DeepMind Control Suite(DMC) tasks ~\cite{tassa2018dmcontrol}. We consider two representative operator families: (i) standard bisimulation operators which defines the bisimulation metric~\eqref{eq:bisim_operator} in DBC~\cite{zhang2021deep}, and (ii) independently-coupled operators shared by MICo~\eqref{eq:mico_operator} and SimSR~\eqref{eq:simsr_operator}. Unless stated otherwise, we keep each baseline pipeline and hyperparameters fixed and only replace the distance term in the representation objective. For the DBC-family experiments, however, we use a deterministic-transition DBC variant when equipping DBC with PAMD, because the original probabilistic DBC loss uses a closed-form Gaussian Wasserstein term tied to a Euclidean ground metric. We therefore report both probabilistic and deterministic DBC baselines, and treat the DBC-Det vs. DBC-Det+PAMD comparison as the controlled comparison for isolating the effect of PAMD. Implementation details are provided in Appendix~\ref{app:ImplementationDetails}. 

\subsection{Plug-in performance across operator families}
\begin{figure*}[t]
  \vskip 0.2in
  \begin{center}
    \centerline{\includegraphics[width=\textwidth]{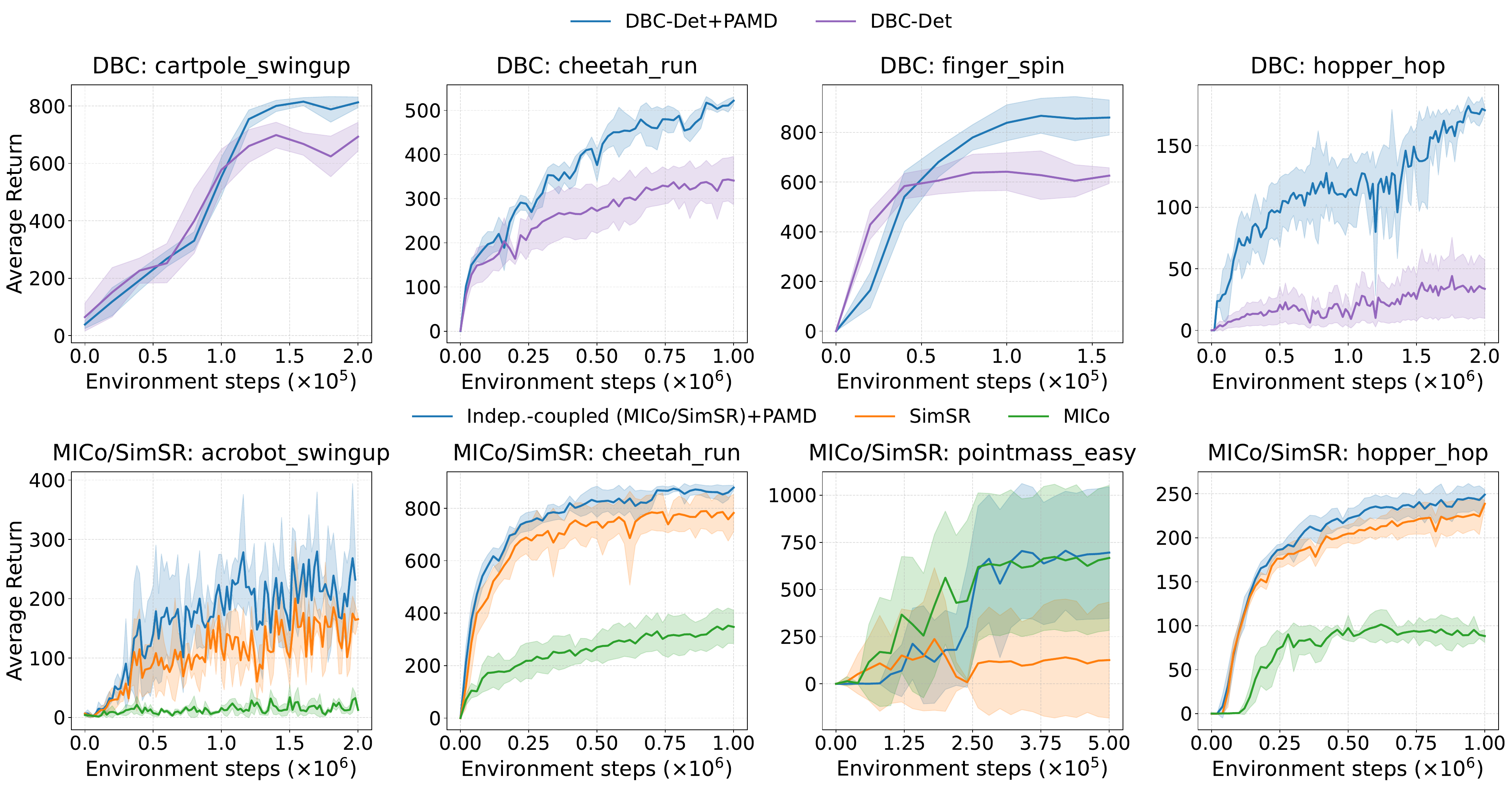}}
    \caption{Results comparing DBC-Det+PAMD to DBC-Det (\textbf{top row}) and the Independently-coupled (MICo/SimSR-style) operator equipped with PAMD to SimSR and MICo baselines (\textbf{bottom row}) across multiple visual MuJoCo tasks. Results show the mean average return over 5 seeds with 1 standard error shaded. For each seed, the average return is computed every 10{,}000 training steps, averaging over 10 episodes. The x-axis denotes the total number of environment interactions, while the y-axis reports the average episodic return.}
    \label{fig:result1}
  \end{center}
\end{figure*}

\paragraph{DBC-style operators.}
We first evaluate PAMD within the DBC family. The original DBC implementation uses a probabilistic Gaussian latent transition model together with a closed-form Gaussian Wasserstein term, whose standard form is tied to a Euclidean ground metric. Directly substituting PAMD into this probabilistic term would therefore change the transition-discrepancy estimator rather than constitute a literal ground-metric drop-in. For this reason, our DBC+PAMD experiments use a deterministic-transition DBC variant.

Table~\ref{tab:dbc_det_comparison} compares probabilistic DBC, deterministic DBC, and deterministic DBC equipped with PAMD, across four pixel-based DeepMind Control Suite tasks
(\texttt{cartpole\_swingup}, \texttt{cheetah\_run}, \texttt{finger\_spin}, and \texttt{hopper\_hop}). The probabilistic and deterministic DBC baselines are broadly comparable in final performance, whereas DBC-Det+PAMD yields substantially higher returns across all evaluated tasks. Thus, for the DBC experiments, the controlled metric-replacement comparison is DBC-Det vs. DBC-Det+PAMD, with probabilistic DBC reported as a reference baseline.

As shown in the \textbf{top row} of Figure~\ref{fig:result1}, the performance gains are pronounced on all tasks, with higher final returns, indicating that the choice of latent distance alone can significantly impact downstream control performance even under an otherwise identical DBC pipeline.
\begin{table}[t]
\centering
\resizebox{\columnwidth}{!}{%
\begin{tabular}{lccc}
\toprule
\textbf{Environment (steps)} & \textbf{DBC (Prob.)} & \textbf{DBC-Det} & \textbf{DBC-Det+PAMD} \\
\midrule
Cartpole Swingup (200k) & $696 \pm 28$ & $710 \pm 55$ & $\mathbf{818 \pm 15}$ \\
Cheetah Run (1M)        & $310 \pm 32$ & $350 \pm 60$ & $\mathbf{524 \pm 14}$ \\
Finger Spin (150k)      & $577 \pm 95$ & $630 \pm 29$ & $\mathbf{860 \pm 81}$ \\
Hopper Hop (2M)         & $67 \pm 16$  & $33 \pm 29$  & $\mathbf{181 \pm 17}$ \\
\bottomrule
\end{tabular}%
} 
\caption{Performance comparison of probabilistic DBC, deterministic DBC, and deterministic DBC equipped with PAMD on image-based DMControl tasks. Results are averaged over 5 random seeds. Bold values indicate the best mean performance for each task.}
\label{tab:dbc_det_comparison}
\end{table}

\paragraph{Independently-coupled (MICo/SimSR-style) operators.}
We next evaluate PAMD under the independently-coupled operator family shared by MICo and SimSR.
Before applying our method, we observe that the relative performance between SimSR and MICo
depends on the task: SimSR outperforms MICo on \texttt{acrobot\_swingup}, \texttt{cheetah\_run},
and \texttt{hopper\_hop}, while MICo performs better on \texttt{pointmass\_easy}.
This highlights that differences in the latent comparator alone already induce
nontrivial performance variation within the same operator family.

Under this independently-coupled operator, replacing the baseline latent comparator with PAMD consistently improves performance across tasks.
As shown in the \textbf{bottom row} of Figure~\ref{fig:result1},
our method yields either clear gains or uniformly stronger performance
relative to both the SimSR and MICo baselines, despite preserving the same fixed-point target and training pipeline.

\paragraph{Ablation on diagonal Mahalanobis variants.}
We next ask whether the full pair-conditioned quadratic form is necessary, or whether the gains of PAMD can be explained simply by learning coordinate-wise rescaling of the latent space. To isolate this effect, we compare PAMD with two intermediate Mahalanobis variants on \texttt{cheetah\_run}: a \emph{global diagonal Mahalanobis} comparator, which learns a single diagonal reweighting shared across all state pairs, and an \emph{adaptive diagonal Mahalanobis} comparator, which predicts pair-dependent diagonal weights but excludes off-diagonal terms. These variants form a natural hierarchy between fixed latent norms and the full PAMD parameterization: fixed norms correspond to a prescribed isotropic geometry, global diagonal Mahalanobis distances allow coordinate-wise anisotropy shared across the dataset, adaptive diagonal Mahalanobis distances allow pair-dependent axis-wise reweighting, and PAMD further permits pair-dependent cross-coordinate interactions through a dense positive-definite matrix. As shown in Figure~\ref{fig:cheetah_ablation}, the diagonal variants can improve over fixed comparators in some settings, but they do not recover the performance of the full PAMD model in either the DBC or independently-coupled operator families. This suggests that the benefit of PAMD is not merely due to adaptive axis-wise scaling; the off-diagonal structure provided by the full Cholesky parameterization offers an additional useful inductive bias while retaining the displacement-based positive-definite quadratic form.
\begin{figure}[t]
    \centering
    \includegraphics[width=\columnwidth]{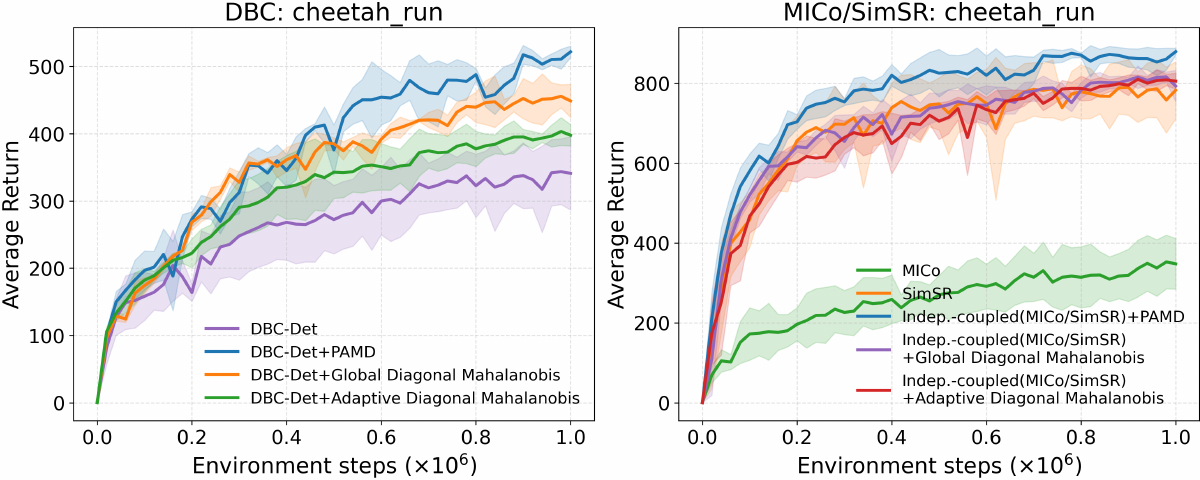}
    \caption{ Ablation results on DMC \texttt{cheetah\_run}. The left and right panels report the DBC and MICo/SimSR families, respectively. Each panel compares PAMD with global and adaptive diagonal Mahalanobis variants. Shaded regions denote $\pm$1 standard error over 5 seeds. }
    \label{fig:cheetah_ablation}
\end{figure}

\paragraph{Ablation on trace normalization.}
We further analyze the role of trace normalization in the proposed distance through a focused ablation study,
reported in Appendix~\ref{app:additional_results}.
We compare variants with and without trace normalization while keeping all other components fixed.
Experiments on DBC (\texttt{finger\_spin}) and SimSR (\texttt{walker\_run}) show that removing trace normalization
consistently degrades performance.
This behavior is consistent with scale degeneracy in the learned quadratic form,
where the fixed-point constraint can be satisfied through uncontrolled scaling rather than meaningful structure.

\paragraph{Experiments with natural-video disturbances.}
To evaluate robustness against task-irrelevant visual information, we replace the default background in DMC \texttt{cheetah\_run} with color-shifted natural-video distractors, as shown in Figure~\ref{fig:cheetah_visual_disturbance}. DBC-Det+PAMD maintains strong performance under this visual disturbance and remains clearly above DBC-Det, whereas the baseline suffers a larger performance drop. These results suggest that PAMD helps the representation focus on behaviorally relevant state information rather than nuisance visual variation.

\begin{figure}[t]
    \centering
    \includegraphics[width=\linewidth]{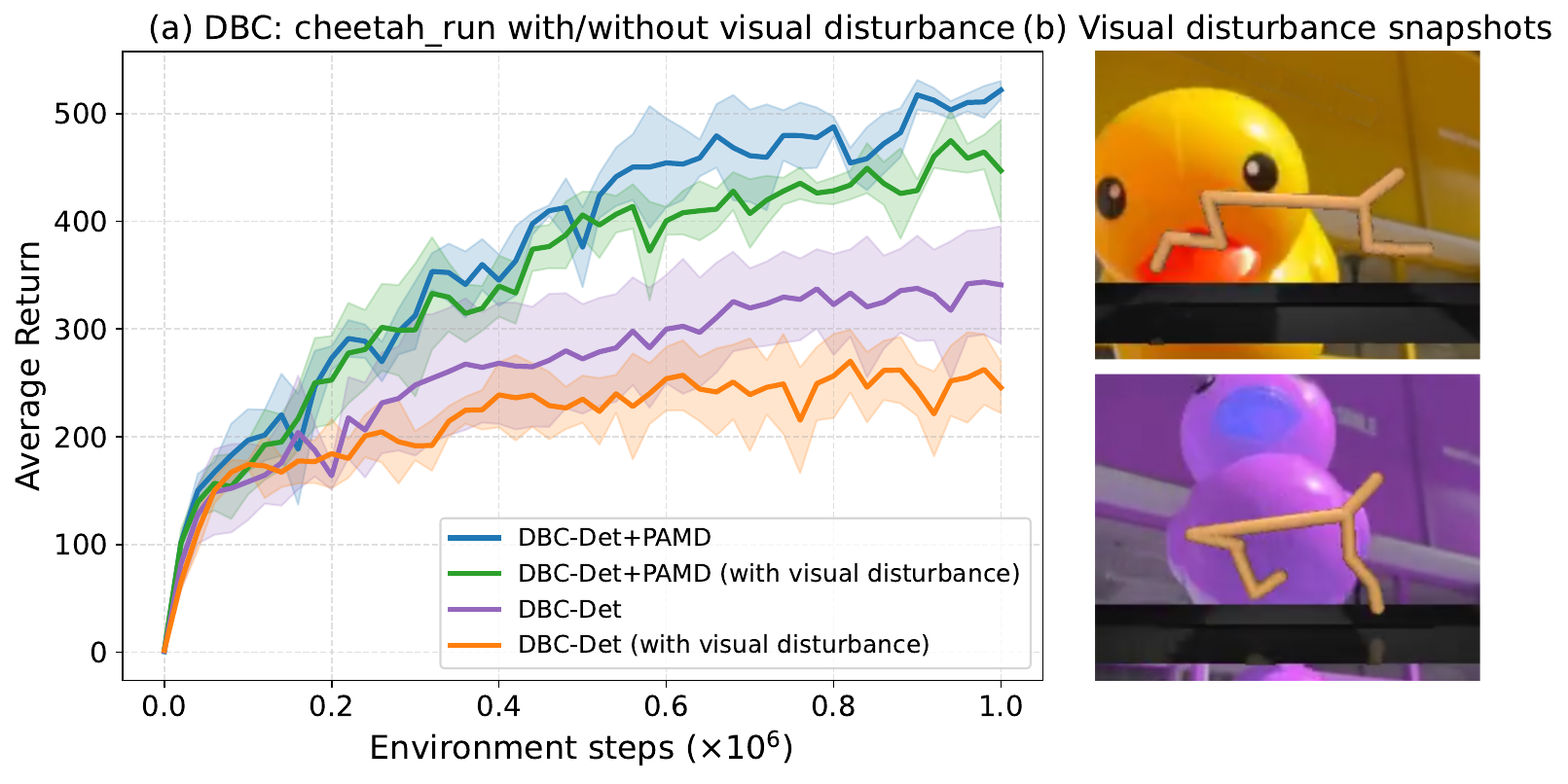}
    \caption{{Natural-video disturbance experiment on DMC \texttt{cheetah\_run}. The default background is replaced with color-shifted natural-video distractors. Curves show the mean return over 5 seeds, with shaded regions denoting $\pm$1 standard error.}}
    \label{fig:cheetah_visual_disturbance}
\end{figure}

\subsection{Why Structure Matters in Pairwise Distances} \label{subsec:structure_matters_pairwise}
A fully learnable pairwise distance $d_\theta$ can improve the metric loss~\eqref{eq:metric_loss} without meaningfully shaping the encoder: it may absorb most of the supervision and satisfy the fixed-point equation~\eqref{eq:generic_behavioral_distance} even when the representation $z=\phi(s)$ is held fixed. This is particularly the failure mode we aim to identify: small metric loss does not necessarily imply a behaviorally organized representation, and in online RL it can correlate with degraded control performance. To directly test whether a distance parameterization requires encoder adaptation, we use an adversarial diagnostic: \emph{residual fitting with a frozen encoder}. If the residual can be driven near-zero while $\phi$ is frozen, then the distance parameterization alone is powerful enough to fit the fixed-point target, and metric loss reduction provides weak pressure to learn an effective encoder. Conversely, if residual reduction is limited under a frozen encoder but improves substantially when $\phi$ is trainable, this indicates that satisfying the fixed-point constraint primarily proceeds via joint representation learning.

We instantiate the SimSR pipeline with either (i) our pairwise-conditioned quadratic-form distance~\eqref{eq:plugin_distance}, or (ii) an \emph{unconstrained} pairwise MLP distance with a comparable parameter budget. All other components and hyperparameters are kept identical.

Figure~\ref{fig:result2} presents:
(\emph{top}) downstream control performance on DMC \texttt{walker\_walk}; and
(\emph{bottom}) the residual-fitting diagnostic based on the independently-coupled one-step fixed-point form,
evaluated offline on a fixed replay buffer.
\begin{figure}[t]
  \vskip 0.10in
  \begin{center}
    \centerline{\includegraphics[width=\columnwidth]{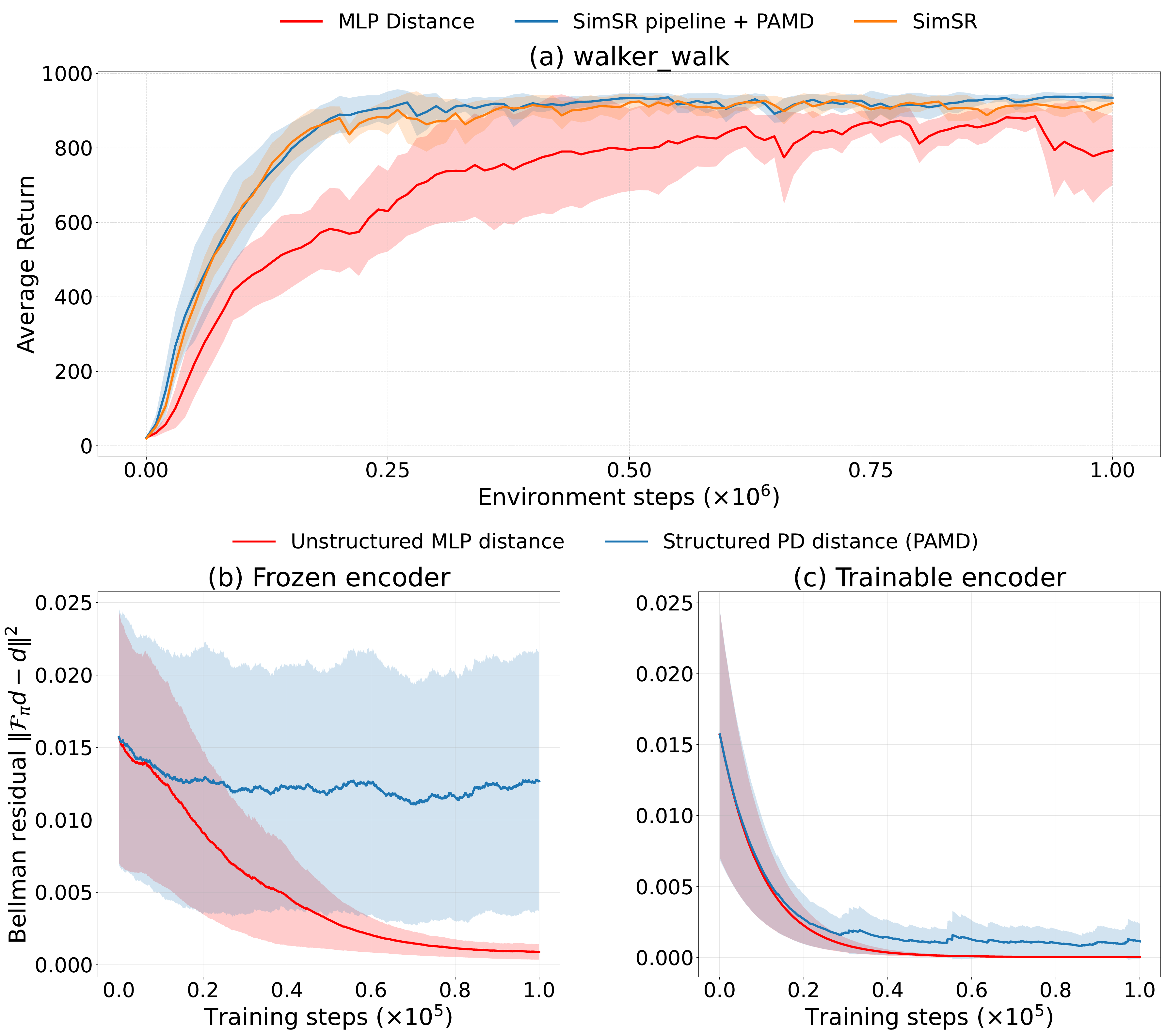}}
    \caption{
    \textbf{Top:} DMC \texttt{walker\_walk} learning curves for SimSR, the SimSR pipeline equipped with PAMD, and SimSR with an unconstrained pairwise-MLP distance (5 seeds; mean $\pm$ 1 s.e.).
    \textbf{Bottom:} Residual-fitting diagnostic (4 seeds) under a \emph{frozen} vs.\ \emph{trainable} encoder, comparing the unconstrained MLP distance and our structured distance.
    }
    \label{fig:result2}
  \end{center}
  \vskip -0.20in
\end{figure}
\paragraph{Residual-fitting diagnostic (offline).}
We first train a SimSR agent on DMC \texttt{walker\_walk} and collect a replay buffer of transitions $(s,r,s^+)$.
Diagnostics are then conducted \emph{offline} on this fixed buffer.
At each update, we sample a minibatch of transitions and form random state pairs within the minibatch.
Let $z=\phi(s)$ and $z'=\phi(s')$ denote latent representations.
We fit the distance module to a one-step fixed-point target using an EMA target network:
\begin{equation} \label{eq:residual_fit}
\begin{aligned}
  \mathcal{L}_{\mathrm{res}}(\theta;\phi) &= \mathbb{E}\!\Big[\big(d_{\theta}(\phi(s),\phi(s'))-\hat d(s,s')\big)^2\Big], \\
  \hat d(s,s') &= \ell(r,r') + \gamma\, d_{\theta^-}(\phi(s^+),\phi(s'^+)),
\end{aligned}
\end{equation}
where $\ell(\cdot,\cdot)$ is a smooth-$\ell_1$ reward distance and $\theta^-$ is a stop-gradient target copy of $\theta$ updated by exponential moving average.
Crucially, we evaluate two settings:
(\emph{i}) \textbf{frozen encoder}: optimize only $\theta$ with $\phi$ fixed;
(\emph{ii}) \textbf{trainable encoder}: optimize $\theta$ and $\phi$ jointly.

\paragraph{Downstream performance (online RL).}
Replacing the latent comparator with an unconstrained pairwise MLP substantially degrades episodic returns on \texttt{walker\_walk},
whereas our structured distance maintains strong performance (Fig.~\ref{fig:result2}, top).
This shows that simply increasing the expressivity of the pairwise distance head is not reliably beneficial in the online RL setting.

\paragraph{Analysis.}
The residual diagnostic explains the above gap. With a \textbf{frozen encoder}, the unconstrained pairwise MLP rapidly reduces $\mathcal{L}_{\mathrm{res}}$ (Fig.~\ref{fig:result2}, bottom), indicating that it can satisfy the one-step fixed-point target largely through distance fitting alone. In contrast, our structured distance yields only limited residual reduction and plateaus above a nonzero level when $\phi$ is frozen, suggesting that it cannot ``absorb'' the supervision without changing the representation. When the encoder is \textbf{trainable}, however, the residual under our structured distance continues to decrease and reaches a level comparable to the unconstrained MLP, indicating that the proposed structure is expressive enough to match the target, but does so primarily via joint adaptation with $\phi$. These results altogether suggest that imposing appropriate structure on the pairwise distance can prevent distance parametrization-dominated solutions and preserve meaningful learning pressure on the encoder, which in turn correlates with improved downstream control.
\section{Conclusion} \label{sec:conclusion}
In this paper, we introduced PAMD: a plug-in pairwise-conditioned positive definite quadratic-form distance structure for bisimulation-based representation learning objectives. We show that PAMD remains comparable to the Euclidean latent distance, and preserves the contraction property of standard behavioral distance operators under fixed-policy MDPs. Empirical validation is conducted on multiple pixel-based continuous-control benchmarks, where PAMD consistently improves downstream control performance when plugged into representative baselines from two major operator families. Finally, an offline residual-fitting diagnostic that contrasts \emph{frozen} versus \emph{trainable} encoders shows that, unlike an unconstrained pairwise MLP distance, PAMD cannot drive the fixed-point residual down by distance fitting alone when the representation is held fixed. Instead, it achieves comparable residual reduction primarily through joint adaptation with the encoder, consistent with its improved downstream control performance.
\section{Limitations} \label{sec:limitations}

A limitation of the current implementation is its dense pair-conditioned quadratic form: although the measured overhead is modest in our DMC setting, constructing a full pair-conditioned positive-definite weight matrix via Cholesky decomposition scales unfavorably with very large latent dimensions. Developing diagonal, low-rank, or structured sparse variants is therefore an important direction for scaling PAMD to higher-dimensional representation spaces.

Another limitation is the scope and protocol matching of our empirical comparisons. Our primary experiments are designed to isolate the effect of replacing the latent comparator inside bisimulation-based visual RL pipelines, and are therefore centered on pixel-based DMControl tasks with fixed baseline pipelines and hyperparameters. The comparison to recent visual RL methods in Table~\ref{tab:recent_visual_rl_comparison} is included only as a reference: those baselines use different training recipes. A fully controlled comparison against stronger recent methods, including TACO, MLR, and DrM, as well as evaluation on broader manipulation-heavy benchmarks such as Meta-World, remains an important direction for future work.

\section{Future Work} \label{sec:future_work}
PAMD is not tied to a single baseline, but it can be used whenever a representation is trained by matching a latent-space dissimilarity to a target notion of behavioral similarity. While we focused on bisimulation-based objectives, an immediate direction is to study how structured, pairwise-conditioned quadratic forms interact with other similarity targets (e.g., value- or advantage-based notions) and other operator choices, and when such structure continues to prevent distance parametrization-dominated fitting while improving downstream performance. More broadly, it would be useful to characterize which classes of behavioral distances are most amenable to such adaptive latent geometries and how the choice of structure trades off expressivity and optimization stability. A complementary theoretical direction is to connect this structured-, adaptive- approach to metric-embedding guarantees by relating the required latent dimension and achievable distortion to properties of the induced quotient space (e.g., covering numbers or effective dimension of $0$-distance classes).
\section*{Acknowledgements}
This work was supported in part by the Institute for Information \& Communications Technology Planning \& Evaluation (IITP) grant funded by the Korea government (MSIT) (No. RS-2019-II190075, Artificial Intelligence Graduate School Program (KAIST)), and in part by the National Research Foundation of Korea (NRF) grant funded by the Korea government (MSIT) (No. RS-2026-25479034).
\section*{Impact Statement}
This work studies fundamental reinforcement learning methods and is primarily methodological in nature. We do not foresee any immediate negative societal impacts arising directly from this research.

\bibliography{references}
\bibliographystyle{icml2026}

%%%%%%%%%%%%%%%%%%%%%%%%%%%%%%%%%%%%%%%%%%%%%%%%%%%%%%%%%%%%%%%%%%%%%%%%%%%%%%%
%%%%%%%%%%%%%%%%%%%%%%%%%%%%%%%%%%%%%%%%%%%%%%%%%%%%%%%%%%%%%%%%%%%%%%%%%%%%%%%
% APPENDIX
%%%%%%%%%%%%%%%%%%%%%%%%%%%%%%%%%%%%%%%%%%%%%%%%%%%%%%%%%%%%%%%%%%%%%%%%%%%%%%%
%%%%%%%%%%%%%%%%%%%%%%%%%%%%%%%%%%%%%%%%%%%%%%%%%%%%%%%%%%%%%%%%%%%%%%%%%%%%%%%
\newpage
\appendix
\onecolumn
\section{Mapping to Prior Theoretical Results} \label{app:mapping_prior_results}
\begin{table}[H]
\centering
\small
\setlength{\tabcolsep}{6pt}
\resizebox{\linewidth}{!}{%
\begin{tabular}{p{0.50\linewidth}cccc}
\toprule
\textbf{Preserved property of the \emph{target} behavioral distance} 
& \textbf{DBC} & \textbf{MICo} & \textbf{SimSR} & \textbf{Ours} \\
\midrule
Existence/uniqueness as a fixed point of a Bellman-style (contractive) operator
& \cmark\ {\small (Thm.~1)}
& \cmark\ {\small (Prop.~4.2, Cor.~4.3)}
& \cmark\ {\small (Thm.~2; same fixed point as MICo)}
& \cmark\ {\small (target distance unchanged)} \\[2pt]

Value-function relation stated \emph{in terms of the target distance}
(e.g., Lipschitz/upper bound on $|V(x)-V(y)|$)
& \cmark\ {\small (cites Ferns'04 Thm.~5)}
& \cmark\ {\small (Prop.~4.8)}
& \cmark\ {\small (Prop.~1; cites MICo)}
& \cmark\ {\small (inherits verbatim)} \\[2pt]

(Pseudo)metric structure of the target distance
(nonnegativity, symmetry, triangle inequality; induces $0$-distance equivalence classes)
& \cmark\ {\small (Thm.~1: fixed point over pseudometrics)}
& \cmark\ {\small (Prop.~4.10)}
& \cmark\ {\small (Thm.~2 + MICo Prop.~4.10)}
& \cmark\ {\small (inherits verbatim)} \\
\bottomrule
\end{tabular}%
}
\caption{\textbf{Preserved theoretical properties.} We modify only the \emph{proxy} latent dissimilarity used to approximate a fixed target behavioral distance; therefore, guarantees stated purely in terms of the target distance (and its defining operator) remain valid.}
\label{tab:preserved_properties}
\end{table}
\section{Geometric Motivation for the Pairwise Adaptive Quadratic Form}
\label{app:geometric_motivation}

A Riemannian metric assigns an inner product to the tangent space at each point
of a manifold, thereby determining how lengths and directions are measured
locally.
This suggests a useful geometric analogy for latent representation learning:
instead of measuring all latent differences with a single fixed norm, one may
allow the geometry of the distance to vary across the representation space.

PAMD follows this motivation only at the level of a local quadratic form. It does not aim to compute geodesic distances or define a full, globally consistent Riemannian manifold structure. Rather, its design is motivated by a local-geometric view of latent representation learning. A fixed latent norm, such as an $\ell_1$ or $\ell_2$ distance, implicitly uses the same geometry at every latent location and along every direction. In visual reinforcement learning, however, the behavioral relevance of latent directions can vary across the representation space: some
directions may mostly encode nuisance visual variation, while others may correspond to changes that are important for rewards or transition behavior.

A useful way to formalize this intuition is to consider a smoothly varying
positive-definite matrix field $H(z) \in \mathbb{S}_{++}^{p}$ on a latent domain
$\mathcal{Z} \subseteq \mathbb{R}^{p}$. Such a field defines the length of a
smooth path $\gamma : [0,1] \to \mathcal{Z}$ by
\begin{equation}
    L_H(\gamma)
    =
    \int_0^1
    \sqrt{
        \dot{\gamma}(t)^\top H(\gamma(t)) \dot{\gamma}(t)
    }
    \, dt ,
\end{equation}
and the corresponding intrinsic distance is
\begin{equation}
    d_H(z,z')
    =
    \inf_{\gamma:\gamma(0)=z,\gamma(1)=z'}
    L_H(\gamma).
\end{equation}
This construction allows the latent geometry to vary across both locations and
directions. However, computing $d_H$ requires solving an optimization problem
over paths and evaluating an integral along the resulting geodesic, which is
impractical inside minibatch Bellman-style representation learning objectives.

For nearby points $z$ and $z'$, a standard local approximation is obtained by
freezing the geometry around a representative point $\bar z$ between $z$ and
$z'$, for example the midpoint. This gives the local quadratic approximation
\begin{equation}
    d_H(z,z')
    \approx
    \sqrt{
        (z-z')^\top H(\bar z) (z-z')
    } .
    \label{eq:local_quad_approx}
\end{equation}
PAMD follows this local-quadratic viewpoint, but replaces the fixed local tensor
$H(\bar z)$ with a learned pair-conditioned positive-definite matrix
$\widetilde G_\theta(z,z')$. In particular, up to the small numerical stabilizer
used in implementation, PAMD measures the displacement
$\Delta = z-z'$ as
\begin{equation}
    d_\theta(z,z')
    =
    \sqrt{
        \Delta^\top \widetilde G_\theta(z,z') \Delta
    } .
\end{equation}
Thus, the learned matrix determines how different latent directions are weighted
for the particular pair being compared.

This interpretation also clarifies the intended scope of the method. Since
$\widetilde G_\theta$ is pair-conditioned rather than defined only as a
single-point tensor field, we do not claim that PAMD induces a globally consistent Riemannian metric tensor or that it computes exact geodesic distances. When the pair-conditioning is implemented symmetrically, this preserves symmetry with respect to the input pair; nevertheless, the resulting quantity is still not
claimed to satisfy global metric properties such as the triangle inequality. Instead, the purpose is to introduce a structured and efficient
local-geometric inductive bias into bisimulation-based representation learning.
Compared with a fixed global norm, PAMD can express pair-dependent anisotropy.
Compared with an unconstrained pairwise MLP distance, it remains tied to the
relative displacement $\Delta$ through a positive-definite quadratic form, which
prevents the distance head from freely absorbing the fixed-point supervision
without shaping the encoder.

From this perspective, the metric-side contribution of PAMD is a practical
middle ground: it brings a local geometric view of adaptive distances into
bisimulation-based visual RL while avoiding the computational burden of full
intrinsic-distance computation.

\section{Additional Experimental Results} \label{app:additional_results}
We present an ablation study analyzing the role of trace normalization in the proposed pairwise-conditioned PD distance. We compare variants \emph{with} and \emph{without} trace normalization, while keeping all other components identical.

Figure~\ref{fig:trace_ablation} reports results on two representative settings. In the left panel, we evaluate DBC and DBC with PAMD on the DMC \texttt{finger\_spin} task. In the right panel, we evaluate SimSR and the SimSR pipeline equipped with PAMD on the DMC \texttt{walker\_run} task.

Across both algorithms, removing trace normalization consistently degrades performance. This behavior is consistent with scale degeneracy in the learned quadratic form, where the distance can satisfy the fixed-point constraint primarily through uncontrolled scaling rather than meaningful geometric structure.

These results support the role of trace normalization as a simple but effective mechanism for stabilizing metric learning and improving downstream control performance.
\begin{figure}[H]
  \centering
  \includegraphics[width=120mm]{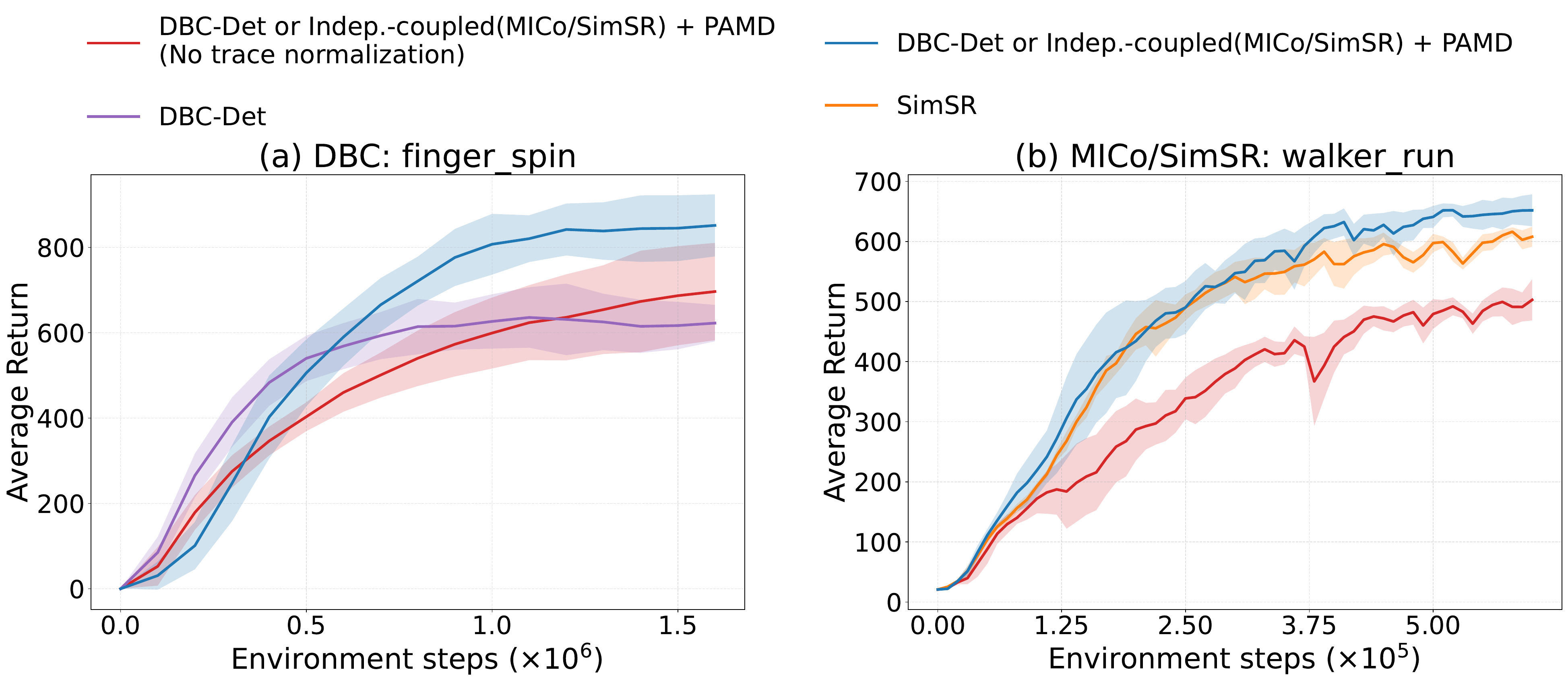}
  \caption{
  \textbf{Trace normalization ablation.}
  \textbf{Left:} DMC \texttt{finger\_spin} learning curves for DBC-Det, DBC-Det+PAMD(with trace normalization),
  and DBC-Det+PAMD(without trace normalization).
  \textbf{Right:} DMC \texttt{walker\_run} learning curves for SimSR, Indep.-coupled+PAMD (with trace normalization),  and Indep.-coupled+PAMD (without trace normalization).
  Curves show mean performance over 5 seeds with 1 standard error shaded.
  }
  \label{fig:trace_ablation}
\end{figure}

\paragraph{Reference comparison to recent visual RL baselines.}
To provide additional context beyond bisimulation-based baselines, we also compare against recent visual reinforcement learning methods on overlapping image-based DMControl tasks. Table~\ref{tab:recent_visual_rl_comparison} reports the corresponding results. This comparison should be interpreted as a contextual reference rather than a fully controlled head-to-head evaluation: PAMD is evaluated as a plug-in module within SAC-based bisimulation pipelines, whereas several compared methods use different training recipes and their reported numbers are taken from prior work. Nevertheless, the comparison shows that the independently-coupled PAMD variant is competitive with stronger recent visual RL baselines on the overlapping tasks, achieving the best reported mean on \texttt{cheetah\_run} and \texttt{walker\_run}, and remaining close to TACO on \texttt{acrobot\_swingup} and \texttt{hopper\_hop}. Fully matched comparisons against recent methods such as TACO, MLR, and DrM under a common protocol remain an important direction for future work.
\begin{table*}[t]
\centering
\resizebox{\textwidth}{!}{%
\begin{tabular}{lcccccccccc}
\toprule
\textbf{Environment} 
& \textbf{\shortstack{Indep.-coupled\\+PAMD}} 
& \textbf{MICo} 
& \textbf{SimSR} 
& \textbf{\shortstack{DBC-Det\\+PAMD}} 
& \textbf{DBC} 
& \textbf{TACO} 
& \textbf{DrQ-v2} 
& \textbf{A-LIX} 
& \textbf{DrQ} 
& \textbf{CURL} \\
\midrule
Acrobot Swingup 
& $\underline{237 \pm 9}$ 
& $15 \pm 9$ 
& $141 \pm 23$ 
& -- 
& -- 
& \cellcolor{gray!15}$\mathbf{241 \pm 21}$ 
& $128 \pm 8$ 
& $112 \pm 23$ 
& $24 \pm 8$ 
& $5 \pm 1$ \\

Cheetah Run 
& \cellcolor{gray!15}$\mathbf{878 \pm 8}$ 
& $347 \pm 63$ 
& $778 \pm 55$ 
& $524 \pm 14$ 
& $350 \pm 49$ 
& $\underline{821 \pm 48}$ 
& $691 \pm 42$ 
& $676 \pm 41$ 
& $474 \pm 32$ 
& $657 \pm 35$ \\

Hopper Hop 
& $\underline{250 \pm 3}$ 
& $78 \pm 17$ 
& $220 \pm 18$ 
& $114 \pm 25$ 
& $15 \pm 9$ 
& \cellcolor{gray!15}$\mathbf{261 \pm 52}$ 
& $189 \pm 35$ 
& $225 \pm 13$ 
& $192 \pm 41$ 
& $152 \pm 34$ \\

Walker Run 
& \cellcolor{gray!15}$\mathbf{652 \pm 24}$ {\scriptsize(500k)} 
& -- 
& $613 \pm 15$ {\scriptsize(500k)} 
& -- 
& -- 
& $\underline{637 \pm 11}$ 
& $517 \pm 43$ 
& $617 \pm 12$ 
& $451 \pm 73$ 
& $387 \pm 24$ \\
\bottomrule
\end{tabular}%
}
\caption{
Reference comparison with recent visual RL baselines on overlapping image-based DMControl tasks. Results for our methods are averaged over 5 random seeds, while the numbers for TACO, DrQ-v2, A-LIX, DrQ, and CURL are taken from prior work and averaged over 6 random seeds. Unless otherwise indicated, results are reported at 1M environment steps, and entries annotated with ``500k'' are evaluated at 500k steps. Bold entries indicate the best mean performance per task, and underlined entries indicate the second-best mean performance.
}
\label{tab:recent_visual_rl_comparison}
\end{table*}
%
% ============================================================
\section{Implementation Details} \label{app:ImplementationDetails}
\subsection{Metric Network Architecture and Construction} \label{app:metricnet}

We parameterize a pairwise-conditioned PD metric with a small MLP, denoted \texttt{MetricNet}, that maps a latent pair $(z,z')$ to a structured matrix $G_\theta(z,z')\succ 0$.

\paragraph{Architecture.}
Let $z,z'\in\mathbb{R}^d$. The network takes the concatenated input $[z;z']\in\mathbb{R}^{2d}$ and outputs $\frac{d(d+1)}{2}$ scalars (the lower-triangular entries of a $d\times d$ matrix):
\[
\mathbb{R}^{2d}\rightarrow 128 \rightarrow 128 \rightarrow \frac{d(d+1)}{2},
\]
using ReLU activations between linear layers.

\paragraph{Positive-definite weight matrix construction via Cholesky decomposition}
The output vector is placed into the lower-triangular part (including the diagonal) of $L_\theta(z,z')\in\mathbb{R}^{d\times d}$, with zeros in the upper triangle. We then form
\[
G_\theta(z,z') \;=\; L_\theta(z,z')\,L_\theta(z,z')^\top,
\]
which is symmetric and positive definite by construction.
\paragraph{Diagonal ReHU constraint (practical stabilization).}
Following the Cholesky convention (positive diagonal for a well-behaved, identifiable factor),
we pass only the diagonal entries of $L_\theta$ through a Rectified Huber (ReHU) nonlinearity with $\eta=1.0$:
\[
\mathrm{ReHU}_\eta(x)=
\begin{cases}
0, & x\le 0,\\
\frac{x^2}{2\eta}, & 0<x<\eta,\\
x-\frac{\eta}{2}, & x\ge \eta.
\end{cases}
\]
This keeps the diagonal nonnegative while remaining smooth near $0$, which empirically reduces
numerical issues when $G$ is later normalized or regularized.
In practice, we add a small constant $\varepsilon > 0$ to the diagonal entries after the ReHU
nonlinearity to ensure strict positive definiteness of $G_\theta$, preventing degenerate cases
with zero eigenvalues during trace normalization or determinant-based regularization.

\paragraph{Input-order symmetry.}
Because the metric should not depend on the order of the pair, we symmetrize:
\[
G_\theta(z,z') \;\leftarrow\; G_\theta(z,z') + G_\theta(z',z).
\]
The resulting matrix is used in the trace-normalized form $\tilde G_\theta$ and the plug-in distance in Eqs.~\eqref{eq:trace_norm}--\eqref{eq:plugin_distance}.

\paragraph{Trace normalization and auxiliary anisotropy term.}
To control the overall scale of the learned pairwise metric, we apply per-pair trace normalization,
\[
\tilde{G}(z,z')
=
\frac{G(z,z')}{\mathrm{tr}(G(z,z'))+\epsilon},
\]
where $\epsilon>0$ is a small constant for numerical stability. This removes the arbitrary scale degree of freedom in the quadratic form, so that the learned matrix primarily determines the relative directional weighting rather than the overall magnitude of the distance.

In addition, we include a small auxiliary term that encourages the normalized metric to deviate from the isotropic metric:
\[
\mathcal{L}_{\mathrm{aux}}
=
-\lambda_{\mathrm{aux}}
\left\lVert
\tilde{G}(z,z')-\tfrac{1}{d}I
\right\rVert_F^2,
\qquad
\lambda_{\mathrm{aux}} = 10^{-3}.
\]
Unlike a conventional isotropic regularizer, this term weakly promotes structured anisotropy in the learned metric. Since $\tilde{G}$ is trace-normalized and positive definite, the term does not encourage unbounded growth of the matrix scale; instead, it biases the metric away from a trivial uniform weighting and toward pair-dependent directional structure.

Finally, in the bisimulation regression objective, the target quantity
\[
\ell(r,r') + \gamma\, d_{\tilde G}(z^+, z^{+\prime})
\]
is treated as a fixed target and detached from the computation graph. As a result, gradients from the residual loss do not propagate through the target branch when fitting the current latent distance, ensuring that the metric and representation are updated by matching a stable Bellman-style target rather than by jointly drifting through the target term.
% ------------------------------------------------------------
\subsection{Computational Overhead and Latent-Dimension Scaling}
\label{app:pamd_compute}

PAMD adds only the pair-conditioned MetricNet branch; the SAC actor--critic, pixel encoder, latent transition model, and other hyperparameters are otherwise kept unchanged. For a minibatch of size $B$, latent dimension $d$, and MetricNet hidden width $h$, the additional branch scales linearly with $B$ and polynomially with $d$. In particular, MetricNet outputs $d(d+1)/2$ lower-triangular entries, and the dense quadratic-form construction introduces higher-order $d$-dependent terms. Under a dense implementation, the dominant asymptotic scaling of the added branch is $O(B(hd^2+d^3))$.

Table~\ref{tab:pamd_flops} summarizes the resulting FLOP count under our default experimental setting, $B=128$, $d=50$, and $h=128$. Under our implementation-level FLOP accounting, the added PAMD branch costs approximately $6.50\times 10^8$ FLOPs per representation update. The shared pixel encoder subtotal is approximately $4.53\times 10^{10}$ FLOPs, so the PAMD branch amounts to about $1.43\%$ of the encoder subtotal. Thus, although PAMD introduces a more expressive metric head, its compute overhead is modest relative to the dominant pixel-encoder computation in our experimental regime.

We also measure end-to-end training throughput while increasing the latent dimension from $d=50$ to $d=200$. As shown in Table~\ref{tab:pamd_wallclock_scaling}, this increase leads to moderate end-to-end slowdown in the tested regime. Across the PAMD variants, total wall-clock time increases by $35.9\%$ for DBC-Det+PAMD and by $7.9\%$ for the independently-coupled PAMD variant; across all reported methods, the total-time slowdown ranges from approximately $1.05\times$ to $1.36\times$. These results indicate that PAMD is not asymptotically free, but that its practical overhead remains manageable for the pixel-based visual-control setting considered in this work.

\begin{table}[t]
\centering
\begin{tabular}{lcc}
\toprule
\textbf{Component} & \textbf{FLOPs / update} & \textbf{Relative to encoder} \\
\midrule
Shared pixel encoder & $4.53\times 10^{10}$ & $100\%$ \\
PAMD MetricNet branch & $6.50\times 10^{8}$ & $1.43\%$ \\
\bottomrule
\end{tabular}
\caption{Approximate computational overhead of the PAMD MetricNet branch under the default setting $B=128$, $d=50$, and $h=128$. FLOPs are reported per representation update. The relative overhead is computed with respect to the shared pixel-encoder subtotal.}
\label{tab:pamd_flops}
\end{table}

\begin{table*}[t]
\centering
\resizebox{\textwidth}{!}{%
\begin{tabular}{lccc ccc ccc}
\toprule
 & \multicolumn{3}{c}{\textbf{Update time (ms)}} 
 & \multicolumn{3}{c}{\textbf{Env. steps/s}} 
 & \multicolumn{3}{c}{\textbf{Total time (min)}} \\
\cmidrule(lr){2-4} \cmidrule(lr){5-7} \cmidrule(lr){8-10}
\textbf{Method} 
& \textbf{50} & \textbf{200} & \textbf{$\Delta$}
& \textbf{50} & \textbf{200} & \textbf{$\Delta$}
& \textbf{50} & \textbf{200} & \textbf{$\Delta$} \\
\midrule
DBC 
& 28.37 & 33.20 & $17.0\%$
& 31.60 & 27.11 & $-14.2\%$
& 52.8 & 61.5 & $16.5\%$ \\
DBC-Det+PAMD 
& 28.96 & 39.98 & $38.0\%$
& 31.14 & 22.92 & $-26.4\%$
& 53.5 & 72.7 & $35.9\%$ \\
MICo 
& 24.73 & 26.11 & $5.6\%$
& 35.54 & 33.87 & $-4.7\%$
& 46.9 & 49.2 & $4.9\%$ \\
SimSR 
& 30.82 & 32.60 & $5.8\%$
& 29.29 & 27.84 & $-5.0\%$
& 56.9 & 59.9 & $5.2\%$ \\
Indep.-coupled+PAMD 
& 36.11 & 39.19 & $8.5\%$
& 25.43 & 23.57 & $-7.3\%$
& 65.5 & 70.7 & $7.9\%$ \\
\bottomrule
\end{tabular}%
}
\caption{Latent-dimension scaling of update time, environment throughput, and total wall-clock training time.
Columns 50 and 200 denote the latent dimension $d$. For each method, $\Delta$ reports the relative change when increasing $d$ from 50 to 200.}
\label{tab:pamd_wallclock_scaling}
\end{table*}
%

% ------------------------------------------------------------
\subsection{MLP Distance Architecture}
\label{app:mlp}

As an unstructured baseline for pairwise distance learning, we consider a simple multilayer perceptron (MLP) that directly maps a pair of latent representations to a scalar distance value.

\paragraph{Distance network.}
Given two latent vectors $z, z' \in \mathbb{R}^d$, the distance is parameterized as
\[
d_\theta(z, z') = f_\theta\!\left([z; z']\right),
\]
where $[\cdot;\cdot]$ denotes concatenation and $f_\theta$ is a three-layer MLP with ReLU activations. Concretely, the architecture is
\[
\mathbb{R}^{2d}
\;\rightarrow\;
h
\;\rightarrow\;
h
\;\rightarrow\;
1,
\]
where $h$ denotes the hidden dimension. We used h=392 to match the total parameter count of \texttt{MetricNet}(See Appendix~\ref{app:param_match}for details). The network outputs a single scalar without any explicit non-negativity, symmetry, or metric constraints.

\paragraph{Usage in representation learning.}
The MLP distance is trained jointly with the encoder using a bisimulation-style objective. Given a minibatch of latent representations $\{z_i\}_{i=1}^B$ and a random permutation $\pi$, we form paired samples $(z_i, z_{\pi(i)})$. The latent distance is computed as
\[
d_\theta(z_i, z_{\pi(i)}),
\]
and the corresponding target is defined by
\[
\lvert r_i - r_{\pi(i)} \rvert
\;+\;
\gamma \,
d_\theta(\hat z_i^{+}, \hat z_{\pi(i)}^{+}),
\]
where $\hat z^{+}$ denotes the predicted next latent state produced by the learned transition model and $\gamma$ is the discount factor.

The distance network and the encoder are optimized by minimizing a smooth $\ell_1$ regression loss between the predicted latent distance and the bisimulation target:
\[
\mathcal{L}_{\text{MLP}}
=
\mathbb{E}\Big[
\operatorname{smooth\_} \ell_1
\big(
d_\theta(z, z'),
\,
\lvert r - r' \rvert + \gamma d_\theta(\hat z^{+}, \hat z'^{+})
\big)
\Big].
\]

\paragraph{Optimization details.}
The parameters of the distance MLP are optimized jointly with the encoder using Adam~\cite{kingma2015adam}, while the target representations used for next-state prediction are computed using an EMA encoder. No additional regularization or structural constraints are imposed on the distance outputs.

% ------------------------------------------------------------
\subsection{MICo Loss Implementation Details} \label{app:mico-impl}

Our MICo implementation strictly follows the original formulation and is integrated into the SAC critic update without modification to the loss definition.

\paragraph{Distance function.}
The MICo distance between two latent states $s$ and $s'$ is defined as
\[
U_\omega(s,s')
=
\frac{\lVert \phi_\omega(s) \rVert_2^2 + \lVert \phi_\omega(s') \rVert_2^2}{2}
+
\beta\,\theta\!\bigl(\phi_\omega(s), \phi_\omega(s')\bigr),
\]
where $\theta(\cdot,\cdot)$ denotes the angular distance computed via cosine similarity:
\[
\theta(z,z')
=
\arctan2\!\bigl(\sqrt{1-\cos^2(z,z')},\;\cos(z,z')\bigr),
\quad
\cos(z,z')=\frac{\langle z,z'\rangle}{\|z\|_2\|z'\|_2}.
\]
We set $\beta=0.1$ in all experiments, consistent with the original MICo setting.

\paragraph{Target construction.}
Given a randomly permuted pair of states $(s,s')$ sampled from the same minibatch, the learning target is defined as
\[
T^U_\omega(s,s')
=
|r_s - r_{s'}|
+
\gamma\,U_\omega(s^+,{s'}^+),
\]
where $s^+$ and ${s'}^+$ denote the corresponding next states.
The target distance is computed using a separate target encoder to ensure training stability.

\paragraph{Loss and optimization.}
The MICo loss minimizes the squared temporal difference:
\[
\mathcal{L}_{\text{MICo}}
=
\mathbb{E}\Bigl[
\bigl(
U_\omega(s,s')
-
T^U_\omega(s,s')
\bigr)^2
\Bigr],
\]
implemented using a smooth $\ell_1$ loss.
This loss is combined with the SAC critic loss as
\[
\mathcal{L}
=
(1-\alpha_{\text{MICo}})\,\mathcal{L}_{\text{critic}}
+
\alpha_{\text{MICo}}\,\mathcal{L}_{\text{MICo}},
\]
where we set $\alpha_{\text{MICo}} = 1\times10^{-5}$ for all pixel-based experiments.

\paragraph{Target encoder update.}
MICo maintains its own target encoder, separate from the SAC target critic. The target encoder is updated using infrequent hard synchronization every $8{,}000$ environment steps, following the original MICo design. All hyperparameters are kept identical to the original MICo setup.

% ------------------------------------------------------------
%
\subsection{Transition Model Details} \label{app:transition}

We follow the latent dynamics modeling choices of the corresponding baselines and do not introduce additional modifications unless explicitly stated.

\paragraph{Ensemble and probabilistic transition models.}

The original DBC implementation uses a probabilistic Gaussian latent transition model. Its bisimulation target computes the transition discrepancy through a closed-form Gaussian Wasserstein term, whose standard form is tied to a Euclidean ground metric. As a result, directly inserting PAMD into this probabilistic term would not be a literal drop-in replacement of the latent comparator; it would require changing the transition-discrepancy estimator itself.

We therefore evaluate the DBC-side PAMD variant under a deterministic-transition DBC formulation. We refer to the corresponding Euclidean baseline as DBC-Det. In this setting, DBC-Det and DBC-Det+PAMD use the same deterministic latent-transition setup, and differ only in the latent comparator used in the representation objective. To control for the transition-model change, Table~\ref{tab:dbc_det_comparison} reports probabilistic DBC and DBC-Det side by side. Their final performances are broadly comparable, while DBC-Det+PAMD substantially improves over both, suggesting that the observed gains are not explained by the transition-model choice alone.

For MICo, no learned transition model is required, as the algorithm directly relies on transitions sampled from the replay buffer.

For SimSR and the SimSR-pipeline implementation of the independently-coupled PAMD variant, we adopt the same \emph{ensemble probabilistic transition model} as used in the original SimSR implementation. Specifically, the latent transition is modeled as an ensemble of $K$ Gaussian predictors
\[
P_k(\phi(s_{t+1}) \mid \phi(s_t), a_t), \quad k=1,\dots,K,
\]
which share the same architecture but are independently initialized and trained. This ensemble structure is retained unchanged when plugging in PAMD.

\paragraph{Sampling procedure.}
The function \texttt{sample\_prediction} follows the baseline implementations. For ensemble-based models, one transition head is uniformly sampled to generate the next latent state.
For deterministic transitions, the predicted next state is used directly.

\paragraph{Transition loss.}
When a probabilistic or ensemble transition model is used, parameters are trained by minimizing the Gaussian negative log-likelihood:
\[
\mathcal{L}_{\text{trans}}
=
\frac{1}{K}
\sum_{k=1}^K
\left[
\frac{1}{2}\log\sigma_k^2(\phi(s_t))
+
\frac{
\bigl(\phi(s_{t+1}) - \mu_k(\phi(s_t), a_t)\bigr)^2
}{
2\sigma_k^2(\phi(s_t))
}
\right].
\]
Here, $K=1$ corresponds to a single probabilistic transition model, while $K>1$ denotes an ensemble of independently parameterized transition models. This loss is identical to that of the original baselines.

\paragraph{Additional details.}
No auxiliary regularization terms or hyperparameters are introduced beyond those already present in the baseline implementations. Except for the explicitly stated DBC deterministic-transition variant, no auxiliary transition-model modifications are introduced beyond those already present in the corresponding baseline implementations.

% ------------------------------------------------------------
\subsection{Encoder Architectures and Actor-Critic} \label{app:arch}

We follow the exact architectural and optimization settings of the corresponding baselines (DBC, SimSR, and MICo where applicable) and do not modify any component other than the proposed distance function.

\paragraph{Encoder.}
All methods use the same convolutional pixel encoder. Given an image observation, we apply a stack of 4 convolutional layers (\texttt{num\_filters}$=32$, kernel size $3$; first layer stride $2$, remaining layers stride $1$) with ReLU activations, flatten the final feature map, and project it to a $d$-dimensional latent via a linear layer. The resulting latent is post-processed differently across baselines: DBC applies LayerNorm to the projected features, SimSR applies $\ell_2$ normalization, and MICo applies LayerNorm (without $\ell_2$ normalization). Convolutional weights are tied across modules when separate encoders are instantiated (e.g., actor/critic).

\paragraph{Actor and critic.}
The policy and value functions are trained using Soft Actor-Critic (SAC). The actor is parameterized as a Gaussian policy with a state-dependent mean and diagonal covariance, where the log-standard-deviation is squashed and constrained within fixed bounds. The critic consists of two independent $Q$-networks, each implemented as a multilayer perceptron with two hidden layers of width $256$. This corresponds to Soft Actor-Critic~\cite{haarnoja2018soft} with the convolutional encoder architecture following Yarats et al.~\cite{yarats2021improving}.

\paragraph{Optimization and training details.}
We use the same SAC optimization procedure, target networks, and update schedules as in the baselines. The encoder is updated jointly with the critic and representation objectives, and target encoders are maintained using EMA. All learning rates, batch sizes, and other hyperparameters are kept identical to the baseline configurations.

\subsection{Training Schedule and Hyperparameters}
Table~\ref{tab:dbc-mico-simsr-hparams} summarizes the training schedule and hyperparameters used across DBC, MICo, and SimSR. Unless explicitly stated, we follow the default settings of the original implementations and keep all shared hyperparameters identical across methods to ensure a fair comparison. In particular, all methods use the same replay buffer capacity, batch size, discount factor, optimizer, learning rates, target-network update frequencies, and latent dimensionality. Differences across methods are limited to algorithm-specific components, such as the MICo loss coefficient, angular weighting, and target encoder update rule, which are set according to the original MICo design.

Table~\ref{tab:action-repeat} reports the action repeat used for each environment. These values follow standard practice in prior visual control benchmarks and match the settings commonly used in the corresponding baselines. For environments evaluated by multiple methods, we use the same action repeat to avoid confounding performance differences due to temporal abstraction.

\begin{table}[H]
\centering
\begin{tabular}{lccc}
\toprule
\textbf{Parameter} & \textbf{DBC} & \textbf{MICo} & \textbf{SimSR} \\
\midrule
Replay buffer capacity & $10^6$ & $10^6$ & $10^6$ \\
Batch size & $128$ & $128$ & $128$ \\
Discount $\gamma$ & $0.99$ & $0.99$ & $0.99$ \\
Optimizer & Adam & Adam & Adam \\

Critic learning rate & $1e{-}3$ & $1e{-}3$ & $1e{-}3$ \\
Actor learning rate & $1e{-}3$ & $1e{-}3$ & $1e{-}3$ \\
Encoder learning rate & $1e{-}3$ & $1e{-}3$ & $1e{-}3$ \\
Decoder learning rate & $1e{-}3$ & -- & $1e{-}3$ \\
Decoder weight decay & $1e{-}7$ & -- & $1e{-}7$ \\

Temperature learning rate & $1e{-}3$ & $1e{-}3$ & $1e{-}3$ \\
Temperature Adam’s $\beta_1$ & $0.9$ & $0.9$ & $0.9$ \\
Initial temperature & $0.1$ & $0.1$ & $0.1$ \\

Critic target update frequency & $2$ & $2$ & $2$ \\
Critic Q-function soft-update rate $\tau_Q$ & $0.005$ & $0.005$ & $0.005$ \\
Critic encoder soft-update rate $\tau_\phi$ & $0.005$ & $0.005$ & $0.005$ \\

Actor update frequency & 2 & $2$ & $2$ \\
Actor log stddev bounds & $[-10, 2]$ & $[-10, 2]$ & $[-10, 2]$ \\

Latent dimension & $50$ & $50$ & $50$ \\

MICo loss coefficient $\alpha$ & -- & $1e{-}5$ & -- \\
MICo angular weight $\beta$ & -- & $0.1$ & -- \\
MICo target encoder update & -- & hard ($8000$) & -- \\
Metricnet hidden dimension & $128$ & $128$ & $128$ \\

\bottomrule
\end{tabular}
\caption{Comparison of hyperparameters used for DBC, MICo, and SimSR.}
\label{tab:dbc-mico-simsr-hparams}
\end{table}

\begin{table}[H]
\centering
\begin{tabular}{lcc}
\toprule
\textbf{Environment} & \textbf{DBC} & \textbf{Indep.-coupled (MICo/SimSR)} \\
\midrule
cartpole, swingup & $4$ & -- \\
cheetah, run & $4$ & $4$ \\
finger, spin & $2$ & -- \\
hopper, hop & $4$ & $4$ \\
acrobot, swingup & -- & $4$ \\
pointmass, easy & -- & $2$ \\
walker, walk/run & -- & $2$ \\
\bottomrule
\end{tabular}
\caption{Action repeat values used for each environment across DBC and independently-coupled (MICo/SimSR) experiments.}
\label{tab:action-repeat}
\end{table}

% ============================================================
% Appendix: Residual-fitting diagnostic & implementation details
% ============================================================
%
\section{Residual-fitting diagnostic: formulation and implementation} \label{app:Residual-fitting diagnostic}
This appendix details the residual-fitting diagnostic used in Fig.~\ref{fig:result2} (bottom). The goal is \emph{not} to train a control agent, but to test a specific failure mode: when the encoder is held fixed, can the \emph{distance module alone} satisfy the independently-coupled (MICo/SimSR style) fixed-point relation?

\paragraph{Why a regression objective appears.}
In independently-coupled fixed-point objectives, the learned distance is intended to approximate a bisimulation quantity defined as the fixed point of a one-step operator. Under a fixed encoder $\phi$, define the bootstrap operator
\begin{equation} \label{eq:app_mico_operator}
  (\mathcal{T}_{\theta^-} d)(s,s') \;=\; \ell(r,r') \;+\; \gamma\, d_{\theta^-}\!\bigl(\phi(s^+),\phi(s'^+)\bigr),
\end{equation}
where $\ell$ is a reward-distance term (we use smooth-$\ell_1$), $\gamma\in(0,1)$ is the discount, and $\theta^-$ denotes a stop-gradient target copy. A distance that perfectly satisfies the one-step fixed-point relation would obey $d_\theta(\phi(s),\phi(s')) = (\mathcal{T}_{\theta^-} d)(s,s')$ on replay samples. Residual-fitting therefore minimizes the squared \emph{fixed-point residual} (a Bellman-style residual for distances):
\begin{equation} \label{eq:app_residual_fit}
  \mathcal{L}_{\mathrm{res}}(\theta;\phi) = \mathbb{E}\Big[ \big( d_{\theta}(\phi(s),\phi(s')) - \underbrace{\ell(r,r') \;+\; \gamma\, d_{\theta^-}(\phi(s^+),\phi(s'^+))}_{\triangleq\ \text{Target}} \big)^2 \Big].
\end{equation}
This objective isolates how much fixed-point consistency can be matched by the distance parameterization itself,
given a frozen representation.

\subsection{Distance parameterizations used in the residual experiment} \label{app:distance_param}
We compare two \emph{pairwise} distance classes in the residual-fitting diagnostic. Both take a latent pair $(z,z')$ as input; all other components of the setup follow the main paper and the appendix implementation details unchanged.

\subsubsection{Unconstrained pairwise MLP distance} \label{app:mlp_distance}
As the unstructured baseline, we use a scalar regressor on the concatenated pair:
\begin{equation} \label{eq:app_mlp_def}
  d_{\theta}^{\mathrm{MLP}}(z,z') = f_{\theta}\bigl([z;z']\bigr),
\end{equation}
where $f_{\theta}$ is a 2-hidden-layer ReLU MLP with a single scalar output. This parameterization imposes no metric constraints (e.g., nonnegativity or symmetry), and in the frozen-encoder regime it can directly fit the fixed-point target by supervised regression on fixed features.

\subsubsection{Structured PD quadratic-form distance (ours)} \label{app:pd_distance}
Our structured distance uses the same PD quadratic-form construction as in our main method:
\begin{equation} \label{eq:app_pd_def}
  d_{\theta}^{\mathrm{PD}}(z,z') = \sqrt{(z-z')^\top\, \tilde{G}_{\theta}(z,z')\, (z-z')}.
\end{equation}
The pairwise-conditioned PD matrix is formed via $G^{(0)}_{\theta}(z,z')=L_{\theta}(z,z')L_{\theta}(z,z')^\top$ and symmetrized as
\begin{equation} \label{eq:app_symmetrize}
  G_{\theta}(z,z') = G^{(0)}_{\theta}(z,z') + G^{(0)}_{\theta}(z',z),
\end{equation}
followed by trace normalization:
\begin{equation} \label{eq:app_trace_norm}
  \tilde{G}_{\theta}(z,z') = \frac{G_{\theta}(z,z')}{\mathrm{tr}(G_{\theta}(z,z'))\;+\varepsilon}.
\end{equation}
All architectural details (lower-triangular parameterization, diagonal \texttt{rehu}) are identical to the MetricNet construction described earlier; here we only change the distance class for the diagnostic.

\subsection{Parameter-matched fairness in the MLP vs.\ PD comparison} \label{app:param_match}
A common concern in comparing distance parameterizations is capacity mismatch. To isolate the effect of \emph{structure} rather than parameter count, we match the total number of trainable parameters in \texttt{DistanceMLP} and \texttt{MetricNet} as closely as possible.

\paragraph{MetricNet parameter count.}
With latent dimension $d=50$, \texttt{MetricNet} predicts the lower-triangular entries of $L\in\mathbb{R}^{d\times d}$, i.e., $k = \frac{d(d+1)}{2} = \frac{50\cdot 51}{2} = 1275$ outputs. Using hidden width $h_{\mathrm{PD}}=128$ (as in the implementation), the three linear layers have sizes
\[
2d \rightarrow h_{\mathrm{PD}} \rightarrow h_{\mathrm{PD}} \rightarrow k.
\]
This yields a total of $193{,}915$ trainable parameters.

\paragraph{DistanceMLP parameter count.}
\texttt{DistanceMLP} uses two hidden layers with width $h_{\mathrm{MLP}}$ and sizes
\[
2d \rightarrow h_{\mathrm{MLP}} \rightarrow h_{\mathrm{MLP}} \rightarrow 1.
\]
We set $h_{\mathrm{MLP}}=392$ in the code, which yields $194{,}041$ trainable parameters. Therefore, the two distance modules are essentially parameter-matched (difference: $126$ parameters), eliminating the trivial explanation that performance differences come from a large capacity gap.

\paragraph{Other controlled factors.}
Beyond parameter count, we keep the following identical across variants in the residual-fitting diagnostic:
\begin{itemize}
  \item the same replay buffer and pairing strategy;
  \item the same encoder architecture and latent dimension $d=50$;
  \item the same optimizer (Adam) and learning rate;
  \item the same discount $\gamma$ and target-network update rule ($\tau$);
  \item identical seed set and logging schedule.
\end{itemize}

\subsection{Experimental details for the residual-fitting diagnostic} \label{app:protocol}
We summarize the protocol corresponding to the provided implementation.

\paragraph{Encoder and preprocessing.}
We used same \texttt{PixelEncoder} used in the experiments, with 4 convolutional layers (32 filters) followed by a linear projection to $d=50$ and LayerNorm.

\paragraph{Optimization and hyperparameters.}
Each run performs a fixed number of gradient updates (\texttt{steps}) on Eq.~\eqref{eq:app_residual_fit}.
We use Adam, discount $\gamma=0.99$. We report results averaged over 4 seeds, and compare frozen vs.\ trainable encoder settings by toggling whether $\phi$ receives gradients.

\subsection{How to interpret the residual curves alongside control performance} \label{app:interpretation}
The diagnostic is most informative when interpreted jointly with downstream control performance (Fig.~\ref{fig:result2}, top).

\paragraph{Frozen encoder.}
With $\phi$ frozen, any reduction in Eq.~\eqref{eq:app_residual_fit} must come from the distance module. If the residual drops substantially, the fixed-point consistency can be matched without modifying representations, which weakens the gradient incentive for updating $\phi$ in end-to-end training.

\paragraph{Trainable encoder.}
If the residual decreases mainly when $\phi$ is trainable, satisfying the fixed-point relation requires adapting the representation, which is consistent with the intended role of bisimulation objectives.

\paragraph{Connection to the main claim.}
In our experiments, the unconstrained MLP can substantially reduce the residual even when $\phi$ is fixed, while the structured PD distance exhibits meaningful reduction primarily when $\phi$ is trainable. This aligns with the control results: the setting that can ``fit the residual'' without changing representations (MLP) correlates with degraded downstream performance, whereas the structured distance maintains strong control performance while preserving representational learning pressure.

%%%%%%%%%%%%%%%%%%%%%%%%%%%%%%%%%%%%%%%%%%%%%%%%%%%%%%%%%%%%%%%%%%%%%%%%%%%%%%%
%%%%%%%%%%%%%%%%%%%%%%%%%%%%%%%%%%%%%%%%%%%%%%%%%%%%%%%%%%%%%%%%%%%%%%%%%%%%%%%

\end{document}